\documentclass{article} 
\usepackage{iclr2027_conference,times}
\usepackage{comment}

\usepackage{amsmath,amsfonts,bm,bbm,mathtools}

\def\eqref#1{equation~\ref{#1}}

\def\1{\bm{1}}

\DeclareMathAlphabet{\mathsfit}{\encodingdefault}{\sfdefault}{m}{sl}
\SetMathAlphabet{\mathsfit}{bold}{\encodingdefault}{\sfdefault}{bx}{n}

\def\gX{{\mathcal{X}}}

\newcommand{\E}{\mathop{\mathbb{E}}\nolimits}   
\renewcommand{\Pr}{\operatorname{Pr}}           

\newcommand{\R}{\mathbb{R}}

\newcommand{\KL}{D_{\mathrm{KL}}}
\newcommand{\Var}{\mathrm{Var}}

\newcommand{\Cov}{\mathrm{Cov}}

\DeclareMathOperator*{\argmin}{argmin}

\usepackage{amsthm}
\theoremstyle{plain}
\newtheorem{definition}{Definition}
\newtheorem{proposition}{Proposition}

\newcommand{\alphabet}{\Sigma}
\newcommand{\alphabar}{\alphabet\cup\{\eos\}}   
\newcommand{\kleene}[1]{#1^{*}}                   

\newcommand{\eos}{\textsc{eos}}

\newcommand{\str}{\boldsymbol{y}}                 

\newcommand{\prefix}{\boldsymbol{y}}              

\newcommand{\plm}{p}
\newcommand{\qlm}{q}
\newcommand{\pprob}{\vec{\plm}}                   
\newcommand{\qpprob}{\vec{\qlm}}                  

\newcommand{\strv}{Y}                             
\newcommand{\yvar}{\boldsymbol{Y}}                

\newcommand{\defeq}{\mathrel{\stackrel{\lower0.5pt\hbox{\fontsize{6pt}{6pt}\selectfont\normalfont def}}{=}}}  

\newcommand{\defn}[1]{\textbf{#1}}

\newcommand{\iid}{\overset{\text{i.i.d.}}{\sim}}

\newcommand{\estimand}{\theta}             
\newcommand{\stot}{S}                 
\newcommand{\horizon}{T}              

\newcommand{\cont}{\boldsymbol{z}}    
\definecolor{prefixgreen}{HTML}{000000}
\newcommand{\xprefix}{\textcolor{prefixgreen}{\boldsymbol{x}}} 
\newcommand{\zprefix}{\textcolor{prefixgreen}{\boldsymbol{z}}} 
\newcommand{\zvar}{\textcolor{prefixgreen}{\boldsymbol{Z}}}   

\newcommand{\simplex}{\triangle}      

\newcommand{\defpropto}{\mathrel{\stackrel{\lower0.5pt\hbox{\fontsize{6pt}{6pt}\selectfont\normalfont def}}{\propto}}} 

\definecolor{rowaccent}{HTML}{000000}
\newcommand{\cardstrut}{\rule[-1.15ex]{0pt}{3.5ex}}
\newcommand{\rowtag}[1]{\cardstrut\rlap{\hspace{0.9em}%
  \smash{\textcolor{rowaccent}{\rule[-0.85ex]{1.4pt}{2.6ex}}\hspace{0.5em}%
  {\scriptsize\sffamily\scshape\textcolor{rowaccent}{#1}}}}}

\newsavebox{\famcardbox}
\newlength{\famcardwidth}
\newcommand{\cardcell}[1]{\sbox0{#1}%
  \ifdim\wd0<\famcardwidth \wd0=\famcardwidth \fi
  \usebox0}

\newcommand{\testf}{f}                
\newcommand{\ppot}{{\vec{\phi}}}      

\newcommand{\cvcoef}{\alpha}          
\newcommand{\astar}{\alpha^{\star}}   

\newcommand{\hest}{\hat\theta}                              

\newcommand{\thmc}{\hat\theta_{\mathrm{MC}}}
\newcommand{\thIS}{\hat\theta_{\mathrm{IS}}}
\newcommand{\thfr}{\hat\theta_{\mathrm{FR}}}               

\newcommand{\cv}{G}                   

\newcommand{\filt}{\mathcal{F}}       

\usepackage{mathtools}
\usepackage{ellipsis}
\makeatletter
\renewcommand*{\mathellipsis}{%
  \mathinner{%
    \kern\ellipsisbeforegap%
    {\ldotp}\kern\ellipsisgap%
    {\ldotp}\kern\ellipsisgap%
    {\ldotp}\kern\ellipsisaftergap%
  }%
}
\renewcommand*{\dotsb@}{%
  \mathinner{%
    \kern\ellipsisbeforegap%
    {\cdotp}\kern\ellipsisgap%
    {\cdotp}\kern\ellipsisgap%
    {\cdotp}\kern\ellipsisaftergap%
  }%
}
\renewcommand*{\@cdots}{%
  \mathinner{%
    \kern\ellipsisbeforegap%
    {\cdotp}\kern\ellipsisgap%
    {\cdotp}\kern\ellipsisgap%
    {\cdotp}\kern\ellipsisaftergap%
  }%
}
\renewcommand*{\ellipsis@default}{%
  \ellipsis@before
  \kern\ellipsisbeforegap
  .\kern\ellipsisgap
  .\kern\ellipsisgap
  .\kern\ellipsisgap
  \ellipsis@after\relax}
\renewcommand*{\ellipsis@centered}{%
  \ellipsis@before
  \kern\ellipsisbeforegap
  .\kern\ellipsisgap
  .\kern\ellipsisgap
  .\kern\ellipsisaftergap
  \ellipsis@after\relax}
\AtBeginDocument{%
  \DeclareRobustCommand*{\dots}{%
    \ifmmode\@xp\mdots@\else\@xp\textellipsis\fi}}
\def\ellipsisgap{-.05em}
\def\ellipsisbeforegap{-.05em}
\def\ellipsisaftergap{.05em}
\makeatother
\DeclareMathOperator{\tr}{tr}                 

\usepackage{xcolor}
\usepackage{microtype}
\definecolor{darkgreen}{RGB}{0,100,0}
\usepackage{mathtools}
\usepackage{hyperref}
\usepackage{url}
\usepackage{graphicx}
\usepackage{booktabs}
\usepackage{natbib}
\usepackage{enumitem}
\usepackage{array}
\usepackage{booktabs}
\usepackage{tabularx}
\usepackage{algorithm}
\usepackage{xcolor}
\usepackage{algorithm}
\usepackage[
    indLines=true,
    noEnd=true,
    spaceRequire=false,
    commentColor=black!55
]{algpseudocodex}

\algrenewcommand\algorithmicrequire{\textbf{Input:}}
\algrenewcommand\algorithmicensure{\textbf{Output:}}

\tikzset{
    algpxIndentLine/.style={draw=black!30, very thin}
}
\usepackage{inconsolata}  
\usepackage{capt-of}  
\usepackage{placeins} 
\usepackage{float}    
\usepackage{wrapfig}   
\usepackage{colortbl}

\usepackage{arydshln} 

\usepackage{textcomp}   

\usepackage{tikz}
\usetikzlibrary{positioning,fit,backgrounds,calc}
\usepackage[most]{tcolorbox}
\tikzset{dtline/.style={draw=black!55, line width=0.6pt}}
\newlength{\dtheady}
\newcommand{\dtsetx}[1]{\pgfmathsetmacro{\dtx}{(#1)==0 ? -18 : (#1)*16-17}}
\newcommand{\dtfull}[1]{%
  \foreach \dtl in {#1}{\ifnum\dtl>-1 \dtsetx{\dtl}
    \draw[dtline] ([xshift=\dtx pt]frame.north west)
      -- ([xshift=\dtx pt]frame.south west);\fi}}
\newcommand{\dthead}[2]{%
  \ifnum#1>0 \dtsetx{#1-1}
    \draw[dtline] ([xshift=\dtx pt]frame.north west)
      -- ([xshift=\dtx pt,yshift=-\dtheady]frame.north west)
      -- ([xshift=#1*16pt-28pt,yshift=-\dtheady]frame.north west);\fi
  \ifnum#2=1 \dtsetx{#1}
    \ifnum#1=0
      \draw[dtline] ([xshift=-12pt,yshift=-\dtheady]frame.north west)
        -- ([xshift=\dtx pt,yshift=-\dtheady]frame.north west)
        -- ([xshift=\dtx pt]frame.south west);
    \else
      \draw[dtline] ([xshift=\dtx pt,yshift=-\dtheady-5pt]frame.north west)
        -- ([xshift=\dtx pt]frame.south west);\fi\fi}
\newcommand{\dtrest}[2]{%
  \ifnum#2=1 \dtsetx{#1}
    \draw[dtline] ([xshift=\dtx pt]frame.north west)
      -- ([xshift=\dtx pt]frame.south west);\fi}
\newcommand{\dtleft}[1]{\ifnum#1=0 7pt\else\dimexpr #1\dimexpr 16pt\relax - 5pt\relax\fi}
\NewTColorBox{dtnode}{m m m}{enhanced jigsaw, breakable, frame hidden,
  interior hidden, boxrule=0pt, boxsep=0pt,
  left=\dtleft{#1},
  right=0pt, top=2pt, bottom=2pt, before skip=0pt, after skip=0pt,
  before upper={\setlength{\parskip}{3pt}\strut},
  overlay unbroken={\dtfull{#2}\dthead{#1}{#3}},
  overlay first={\dtfull{#2}\dthead{#1}{#3}},
  overlay middle={\dtfull{#2}\dtrest{#1}{#3}},
  overlay last={\dtfull{#2}\dtrest{#1}{#3}}}
\newcommand{\dtdown}[1]{\dtsetx{#1}
  \draw[dtline] ([xshift=\dtx pt]frame.north west)
    -- ([xshift=\dtx pt,yshift=-\dtheady+5pt]frame.north west);}
\newcommand{\dtbottomlabel}[1]{\if\relax\detokenize{#1}\relax\else\dtsetx{0}
  \node[dttagstyle, rotate=90, anchor=west]
    at ([xshift=\dtx pt,yshift=4pt]frame.south west) {#1};\fi}
\NewTColorBox{dtstack}{m m m}{enhanced jigsaw, breakable, frame hidden,
  interior hidden, boxrule=0pt, boxsep=0pt,
  left=\dtleft{#1},
  right=0pt, top=2pt, bottom=2pt, before skip=0pt, after skip=0pt,
  before upper={\setlength{\parskip}{3pt}\strut},
  overlay unbroken={\dtfull{#2}\dtdown{#1}\dtbottomlabel{#3}},
  overlay first={\dtfull{#2}\dtdown{#1}},
  overlay middle={\dtfull{#2}},
  overlay last={\dtfull{#2}\dtbottomlabel{#3}}}
\tikzset{dttagstyle/.style={draw=black!70, fill=white, line width=0.5pt,
  inner sep=1.3pt, minimum width=22pt, font=\scriptsize\sffamily\bfseries}}

\newlength{\dttaggutter}
\newlength{\dtcardextra}
\newsavebox{\dtcardbox}
\newcommand{\dtcard}[2]{\par\vspace{2pt}\noindent
  \sbox{\dtcardbox}{\footnotesize
    \setlength{\tabcolsep}{0pt}\renewcommand{\arraystretch}{1.25}%
    \setlength{\aboverulesep}{0pt}\setlength{\belowrulesep}{0pt}%
    \begin{tabular}{#1}#2\end{tabular}}%
  \ifdim\wd\dtcardbox<\dimexpr\linewidth+\dtcardextra-\dttaggutter\relax
    \makebox[\dimexpr\linewidth+\dtcardextra-\dttaggutter\relax][c]{\usebox{\dtcardbox}}%
  \else
    \usebox{\dtcardbox}%
  \fi\par}
\newcolumntype{V}{!{\vrule width 0.5pt}}

\newcommand{\dtcardhead}[1]{{\scriptsize\sffamily\scshape\textcolor{rowaccent}{#1}}}

\usepackage[disable,textsize=footnotesize,colorinlistoftodos]{todonotes}
\newif\ifshowtodos\showtodosfalse
\usepackage{hyperref}
\definecolor{linkblue}{RGB}{0,60,160}
\definecolor{viceviolet}{HTML}{5B2A86}
\definecolor{viceteal}{HTML}{00A3A8}
\definecolor{vicepink}{HTML}{E4007C}
\hypersetup{colorlinks=true,linkcolor=viceviolet,citecolor=viceteal,urlcolor=vicepink}
\usepackage{fontawesome5} 
\usepackage{needspace}
\usepackage{cleveref}
\usepackage{etoc} 
\crefname{section}{\S}{\S\S}
\crefname{table}{Tab.}{Tab.}
\crefname{figure}{Fig.}{Figs.}
\crefname{algorithm}{Alg.}{}
\crefname{equation}{Eq.}{Eq.}
\crefname{appendix}{App.}{}
\crefname{theorem}{Theorem}{}
\crefname{prop}{Proposition}{}
\crefname{definition}{Def.}{}
\crefname{cor}{Corollary}{}
\crefname{observation}{Observation}{}
\crefformat{section}{\S#2#1#3}
\Crefformat{section}{\S#2#1#3} 
\crefname{namedtheorem}{Hyp.}{Hypotheses}
\crefname{footnote}{footnote}{footnotes}
\Crefname{footnote}{Footnote}{Footnotes}

\usepackage{xcolor}
\usepackage{cancel}

\newcounter{cmtno}
\newcommand{\cmark}[2]{
  \refstepcounter{cmtno}%
  \ifshowtodos
  \rlap{\smash{\raisebox{1.4ex}{\tikz[baseline]{\node[fill=#1,text=white,inner sep=0.8pt,%
    rounded corners=1pt,font=\tiny\bfseries,anchor=base]{#2\thecmtno};}}}}%
  \fi
}

\newcommand{\timv}[1]{\cmark{darkgreen!70}{T}\todo[color=darkgreen!25,size=\tiny]{\textbf{[T\thecmtno] timv:} #1}}
\newcommand{\Timv}[1]{\todo[inline,color=darkgreen!25]{\textbf{timv:} #1}}
\newcommand{\shubhangi}[1]{\cmark{orange!80}{S}\todo[color=yellow!25,size=\tiny]{\textbf{[S\thecmtno] shubhangi:} #1}}

\newcommand{\response}[1]{\textbf{#1:}}

\newcommand{\claude}[1]{#1}

\renewcommand{\hat}[1]{\widehat{#1}}

\title{Estimating Great Expectations under Autoregressive Language Models with Potentials}

\author{Francesco Ignazio Re$^{1}$ \quad Shubhangi Ghosh$^{2}$ \quad Tim Vieira$^{1}$ \quad Ryan Cotterell$^{1}$ \\
$^{1}$ETH Z\"urich \quad $^{2}$Columbia University \\
\texttt{\href{mailto:francesco.re@inf.ethz.ch}{francesco.re@inf.ethz.ch}} \quad
\texttt{\href{mailto:sg4156@columbia.edu}{sg4156@columbia.edu}} \\
\texttt{\href{mailto:tim.vieira@gmail.com}{tim.vieira@gmail.com}} \quad
\texttt{\href{mailto:ryan.cotterell@inf.ethz.ch}{ryan.cotterell@inf.ethz.ch}}}

\iclrfinalcopy 
\begin{document}

\maketitle
\lhead{Preprint} 

\begin{abstract}
Many applications of language models require estimating the expectation of a test functional under the model. Obtaining precise estimates can be computationally expensive. We show how to improve estimation precision by exploiting next-token conditional probabilities available as a by-product of sampling. Our approach uses potentials: real-valued functions of prefixes that induce an additive decomposition of the test functional. We construct a family of estimators parameterized by these potentials and derive conditions for variance reduction. Across applications including entropy estimation, watermark detection, and rare-event probability estimation, our estimators achieve substantial variance reductions relative to Monte Carlo and importance sampling baselines at comparable computational cost.
\par\smallskip
\claude{\centering\href{https://anonymous.4open.science/r/great_expectations_lm_potential-0DF4/}{\faGithub\ Code Repository}\par}
\end{abstract}

\etocsettocdepth.toc{none}

\section{Introduction}
\label{sec:intro}

A language model is a probability distribution over strings.
The most
common application of such a model is to generate free-form text from a prompt
\citep{radford2019language,brown2020language}.
However, for some applications, the quantity of
interest is not the sample itself but rather the expectation of a test functional under the
model.
For instance,
\begin{itemize}[leftmargin=*,nosep]
  \item In reinforcement learning from human feedback, proximal policy
  optimization \citep[PPO;][]{schulman2017proximal} regularizes the policy
  by penalizing its Kullback--Leibler divergence from a reference model
  \citep{christiano2017deep,stiennon2020learning,ouyang2022training}, a
  divergence that is itself an expectation under the policy.
  \item In AI safety, one is often interested in the probability that a model
  emits a harmful or otherwise undesirable output \citep{jones2025forecasting,angell2026estimating}.
  \item In distillation, one fits a simpler surrogate, such as an $n$-gram model,
  to a large language model by matching expectations such as the cross-entropy
  under the teacher \citep{hinton2015distilling, agarwal2024policy}.
  \item In language-model watermarking, one may be interested in the probability that a generated text is
detected as watermarked; that can be viewed as an expectation under the model's 
distribution \citep{kirchenbauer2023watermark}.
  \item In interpretability, distributional functionals quantify how much a
  given context shifts a model's predictions
\citep{pezeshkpour2023measuring,du2024context,stoehr2024activation}.
\end{itemize}

More precisely, given a language model $\plm$ and a test function $\testf$, we wish to compute the expectation of $\testf$ under $\plm$. Unless $\plm$ and $\testf$ exhibit exploitable structure\footnote{When $\plm$ is a probabilistic finite-state automaton \citep{rabin1963probabilistic} or a probabilistic context-free grammar \citep{booth1973applying} and $\testf$ is additively decomposable, one can compute such expectations with dynamic programming~\citep{eisner2002parameter,li2009first,zmigrod2021efficient}.}, this expectation must be approximated. \emph{Monte Carlo methods} offer a well known approach: sample strings, evaluate $\testf$ on them, and average the results. Drawing independent strings directly from $\plm$ yields the standard Monte Carlo estimator; \shubhangi{this phrase about the proposal seems awkward. I would add it at the end of the paragraph, where maybe we can have one sentence about the existing variance reduction frameworks before introducing ours.}drawing instead from a proposal distribution and reweighting by the corresponding likelihood ratio yields \emph{importance sampling}. These estimators are unbiased but can have high variance, so achieving a desired level of precision may require generating many samples at substantial computational cost.

In this paper, we introduce a variance-reduction framework that exploits the next-token distributions computed as a by-product of sampling. The idea is to write $\testf$ as an additive decomposition of token-level increments along the string and to replace each increment with its conditional expectation given the prefix, taken over the next token. Applied to an existing unbiased estimator, this procedure always preserves unbiasedness while making the variance depend on the decomposition.
The choice of decomposition is thus the central design choice in the framework.

For a given $\testf$, we parametrize the family of its additive decompositions using \emph{potentials}: real-valued functions on prefixes. We borrow this construction from potential-based reward shaping in reinforcement learning, where potentials redistribute rewards along a trajectory to facilitate learning \citep{ng1999policy}. Here, we repurpose potentials to redistribute the function value over marginalizable token level increments and refer to the resulting estimator as the \emph{potential estimator}. \looseness=-1

We first show that there exists an oracle potential yielding a zero-variance estimator. Since it is generally intractable, we derive an exact criterion for when an arbitrary potential reduces variance. We then use the oracle potential to guide potential design: when 
$\testf$ has known additive structure, we use  components of the oracle potential that can be computed directly from 
$\testf$; for arbitrary functions over strings, we approximate the oracle potential by learning it from samples. In general, given a set of candidate potentials, we introduce a potential-selection procedure that preserves unbiasedness while using the full sample.

Our experiments demonstrate substantial gains in statistical efficiency across four applications and several estimands. For losses and divergences, simple potentials reduce variance in all tested configurations, yielding up to \(17\times\) the effective sample size of Monte Carlo. Their gradient counterparts also improve on Monte Carlo in every evaluated comparison, with variance reductions of up to \(89\%\). {For expected \(n\)-gram counts and frequencies, our estimator produces even larger gains: about \(99\%\) lower variance than Monte Carlo after accounting for computational cost (median over targets with lower estimation uncertainty).}  For non-additive event probabilities, approximations to the oracle potential are also effective: our estimator reduces the variance of watermark-detection estimates by a median \(36\%\), while learned potentials selected through nested cross-fitting reduce the variance of the importance-sampling estimator introduced by \citep{dorman2026rare} by a median \(12\%\). In the latter experiment, this corresponds to approximately \(900{,}000\) additional baseline generations. Overall, our results establish potential design as a practical strategy for improving expectation estimates after sampling, across both direct Monte Carlo and importance sampling. We discuss related work in \Cref{app:related}. \looseness=-1

\section{Expectations under Autoregressive Language Models}
\label{sec:prelim}


\subsection{Language Models}
\label{sec:prelim:lm}

Let $\alphabet$ be a finite, non-empty set of symbols, and let $\eos \notin \alphabet$ be a distinguished \defn{terminal symbol}. A
\defn{string} is a finite sequence of symbols from $\alphabet$. Let
$\kleene{\alphabet}$ denote the (countably infinite) set of all such
strings. For $\str,\str' \in \kleene{\alphabet}$, we write $\str \succeq \str'$
when $\str'$ is a prefix of $\str$, and $\str' \prec \str$ when $\str'$ is a strict
prefix of $\str$. 

A \defn{language model} $\plm$ is a
probability distribution over complete strings taking values in $\kleene{\alphabet}$, i.e., $\plm \in \simplex(\kleene{\alphabet})$, where $\simplex(\gX)$ denotes the set
of all probability distributions over a countable set $\gX$.
We write $\yvar \sim \plm$ for a random string drawn from $\plm$, so that
$\yvar \succeq \prefix$ is the event that $\yvar$ extends $\prefix$.
For any distribution $\plm$, we use the following notation for the \defn{prefix probability} of a string $\prefix \in \kleene{\alphabet}$, which gives the marginal probability that a string sampled from $\plm$ extends $\prefix$:
\begin{equation}\label{eq:prefixprob}
\pprob(\prefix) \defeq \Pr[\yvar \succeq \prefix] =\sum_{\cont \in \kleene{\alphabet}} \plm(\prefix\cont).
\end{equation}
The \defn{conditional prefix probability}\timv{It's worth writing the conditional prefix probability as a conditional Pr expression given that we bothered to do that above. {\textbf{reply:} you mean like this?}} is then {$\pprob(\cont \mid \prefix) \defeq \Pr[\yvar \succeq \prefix\,\cont \mid \yvar \succeq \prefix] = \pprob(\prefix\,\cont)/\pprob(\prefix)$}, and we extend
$\pprob(\cdot \mid \prefix)$ to the terminal symbol $\eos$ by $\pprob(\eos \mid \prefix) \defeq
\plm(\prefix)/\pprob(\prefix)$. Whenever $\pprob(\prefix) = 0$, we use the convention
that $\pprob(\cont \mid \prefix) \defeq 0$ for every $\cont \in \alphabet$ and
$\pprob(\eos \mid \prefix) \defeq 1$, so that $\pprob(\cdot \mid \prefix)$ remains a
valid conditional distribution over $\alphabar$. \timv{it's honestly not as simple as the chain rule as that assumes a fixed collection of random variables, so I would rephrase this as being a valid factorization of the probability of any given string \response{francesco} do you think this works?} Using the conditional prefix probabilities, the probability of any string  $\str = y_{1}\cdots y_{|\str|} \in \kleene{\alphabet}$ under a language model admits a left-to-right autoregressive factorization:
\begin{equation}\label{eq:lm-chainrule}
\plm(\str)
=
\pprob(\eos \mid \str)\,
\prod_{n=1}^{|\str|}
\pprob(y_n \mid \str_{<n}),
\end{equation}
where $\str_{<n} \defeq y_{1}\cdots y_{n-1}$. Each $\pprob(\cdot \mid \str_{<n})$ in \cref{eq:lm-chainrule} is a probability distribution over $\alphabet \cup \{\eos\}$.

\subsection{Problem Statement}
\label{sec:prelim:problem}
Let $\plm \in \simplex(\kleene{\alphabet})$ be a language model and let
$\testf \colon \kleene{\alphabet} \to \R$\timv{what about vector-valued functions?  Our methods work fine for them with minor tweaks, should we say that?  If you want to talk about estimating a gradient, you will need that ability for example!} be a real-valued \defn{test function}. We study the problem of estimating
\begin{equation}\label{eq:target-string}
\estimand \defeq \E_{\yvar \sim \plm}[\testf(\yvar)] = \sum_{\str \in \kleene{\alphabet}} \plm(\str)\testf(\str),
\end{equation}
assuming throughout that \(\estimand\) is finite. For modern language
models, this expectation is rarely available in closed form. The string space
$\kleene{\alphabet}$ is countably infinite, and summing exactly over all strings is generally
infeasible. A standard solution is to approximate such expectations using Monte Carlo methods. We review the canonical Monte Carlo estimator and importance sampling in the remainder of this section. In the next section, we introduce potential-based estimation, an additional estimation stage that can reduce the variance of these estimators after generation by reusing the sampled strings and the model’s next-token distributions retained at sampling time.

\subsection{Monte Carlo Estimation}
\label{sec:prelim:mc}

The most direct estimator of $\estimand$ is an empirical average over sampled strings. We assume throughout the paper that $\E_{\plm}\!\left[\testf(\yvar)^2\right] < \infty$ \timv{and finite expectation? \response{francesco}Yes I agree this is clearer. } and let $\yvar^{(1)},\ldots,\yvar^{(M)} \iid \plm $ . The \defn{Monte Carlo estimator} and its variance are\timv{the expression for variance is a little weird.  Isn't it better to write it as $\frac{1}{M}(E[f^2] - \theta^2)$.  You can even write the finite variance conditions more precisely by saying finite first and second moments $E[f]$ and $E[f^2]$. (Ofc, you should use the more precise notation, rather than my shorthand.)}
\begin{equation}\label{eq:mc-variance-string}
\thmc \defeq \frac{1}{M}\sum_{m=1}^{M} \testf(\yvar^{(m)}),
\qquad
\Var_{\plm}(\thmc) = \frac{1}{M}
\left(
\E_{\plm}\!\left[\testf(\yvar)^2\right]
-
\estimand^2
\right).
\end{equation}
Because the samples are drawn i.i.d.\ from $\plm$ and $\estimand$ is finite, $\thmc$ is unbiased, $\E[\thmc] = \estimand$. Moreover, the strong law of large numbers guarantees that $\thmc$ is \emph{strongly consistent}: it converges almost surely to $\estimand$ as $M \to \infty$ \citep{owen2013monte,robert2004monte}.



\subsection{Importance Sampling}
\label{sec:importance-sampling}

Importance sampling estimates an expectation under $\plm$ using samples from a proposal distribution $\qlm$.  {Assume that $\qlm(\str)>0$ whenever $\plm(\str)|\testf(\str)|>0$. For $\str\in\kleene{\alphabet}$, define $f_q(\str)\defeq \plm(\str)\testf(\str)/\qlm(\str)$ when $\qlm(\str)>0$, and $f_q(\str)\defeq0$ otherwise.} Then $\estimand=\E_{\yvar\sim\plm}[\testf(\yvar)]=\E_{\yvar\sim\qlm}[f_q(\yvar)]$.

Thus, for $M$ independent samples $\yvar^{(1)},\ldots,\yvar^{(M)}\iid\qlm$, the \defn{importance-sampling estimator} is simply the Monte Carlo estimator of $f_q$ under $\qlm$, such that
\begin{equation}
\label{eq:is-estimator}
\thIS
\defeq
\frac{1}{M}\sum_{m=1}^{M}\frac{\plm(\yvar^{(m)})}{\qlm(\yvar^{(m)})}\testf(\yvar^{(m)}),
\qquad
\Var_{\qlm}(\thIS)=\frac{1}{M}\bigl(\E_{\qlm}[f_q(\yvar)^2]-\estimand^2\bigr).
\end{equation}

This estimator is unbiased for every such proposal distribution $\qlm$, but its variance may be lower or higher than that of $\thmc$ in \Cref{eq:mc-variance-string} \citep{robert2004monte,owen2013monte}. Its variance is small when the weighted contribution $\tfrac{\plm(\yvar)}{\qlm(\yvar)}\testf(\yvar)$ varies little under $\qlm$. When $\E_{\plm}[|\testf(\yvar)|]>0$, the variance-minimizing proposal is $\qlm^\star(\str)\defpropto\plm(\str)|\testf(\str)|$, which assigns probability in proportion to each string’s absolute contribution to the expectation and makes the magnitude of the weighted contribution constant. \Cref{prop:is-variance} makes this precise.

\section{Estimating Expectations with Potentials}
\label{sec:sequential-decompositions}
We seek to construct a lower-variance estimator by exploiting the target model's next-token distributions along sampled strings, which are already computed when sampling from $\plm$ or evaluating $\plm(\yvar)$ for importance weighting under a proposal $\qlm$. For notational simplicity, we assume $\qlm=\plm$ throughout this section; the general case $\qlm\neq\plm$ is given in \Cref{app:importance-sampling} (\nameref{app:importance-sampling}).

\subsection{Additive decompositions over prefixes: potentials}
Our construction begins with a potential, a prefix-level function $\ppot\colon \kleene{\alphabet}\to\R$ assigning a real value to each prefix. We use its increments over successive prefixes to define token-level contributions along each sampled string. Specifically, for the fixed test function $\testf$, we define the increment operator $\Delta$, which maps a potential $\ppot$ to its increment function $\Delta\ppot$, given by
\begin{equation}
\label{eq:increment}
\Delta\ppot(a\mid\prefix)\defeq\ppot(\prefix a)-\ppot(\prefix)\quad(a\in\alphabet),
\qquad
\Delta\ppot(\eos\mid\prefix)\defeq\testf(\prefix)-\ppot(\prefix).
\end{equation}
For notational ease, we write $\Delta\ppot(a\mid\prefix)$ for
$[\Delta\ppot](a\mid\prefix)$, and use the same convention whenever an
operator acts on a function.
Using this notation, given the random string $\yvar \sim \plm$ and its  length $|\yvar|$, we can rewrite $f(\yvar)$ as
\begin{equation}
\label{eq:additive-decomposition}
\testf(\yvar)
=
\ppot(\varepsilon)
+
\sum_{t=1}^{|\yvar|}
\Delta\ppot(\strv_t\mid\yvar_{<t})
+
\Delta\ppot(\eos\mid\yvar).
\end{equation}

This representation expresses $\testf(\yvar)$ as a sum of random token-level increments $\Delta\ppot(\strv_t \mid \yvar_{<t})$ along the sampled string $\yvar$. At each realized prefix $\str_{<t}$, we already have the next-token distribution $\pprob(\cdot \mid \str_{<t})$, which provides the probabilities needed to compute the next increment's conditional expectation by averaging over possible next tokens. In the next subsection, we show how such expected value can be used for variance reduction.


\subsection{The Potential Estimator}

We use the \defn{Bellman operator} to map a potential to its expected value over the next token:\timv{You need to explain that $\mathbf B\ppot(\prefix)$ is supposed to be read as $[\mathbf B\ppot](\prefix)$ rather than the more obvious $\mathbf B[\ppot(\prefix)]$, i.e., you are deliberately messing with the default operator precedence to be more concise.  The Bellman operator isn't the only place that you are doing this, so be upfront about where and why you are doing it. \response{francesco}  I introduce the increment as an operator in equation (6). When you said that we are not doing it only with the Bellman operator I assumed you meant we are also doing it with the increment $\Delta$? }
\begin{equation}
\mathbf B\ppot(\prefix)
\defeq
\sum_{a\in\alphabet}
\pprob(a\mid\prefix)\,\ppot(\prefix a)
+
\pprob(\eos\mid\prefix)\,\testf(\prefix).
\end{equation}

The \emph{expected potential increment} at prefix $\prefix$ is then
\begin{equation}
\label{eq:bellman-averaged-increment}
\mathbf B\Delta\ppot(\prefix)
\defeq
\E_{A\sim\pprob(\cdot\mid\prefix)}
\bigl[\Delta\ppot(A\mid\prefix)\bigr]
=
\mathbf B\ppot(\prefix)-\ppot(\prefix).
\end{equation}
Substituting these expectations in \cref{eq:additive-decomposition} gives the following
estimator. 

\begin{definition}[Potential estimator]
\label{def:potential-estimator}
Let $\yvar^{(1)},\ldots,\yvar^{(M)}\iid\plm$. Given a potential $\ppot$, define
\begin{equation}
\label{eq:potential-estimator}
\hest(\ppot)
\defeq
\ppot(\varepsilon)
+
\frac{1}{M}
\sum_{m=1}^{M}
\sum_{t=1}^{|\yvar^{(m)}|+1}
\mathbf B\Delta\ppot(\yvar^{(m)}_{<t}),
\qquad
\Var_{\plm}\bigl(\hest(\ppot)\bigr)
=
\frac{1}{M}
\Var_{\plm}\!\left(
\sum_{t=1}^{|\yvar|+1}
\mathbf B\Delta\ppot(\yvar_{<t})
\right).
\end{equation}
The inner sum includes the final prefix $\yvar^{(m)}$, at which $\eos$ is drawn.
\end{definition}

We assume that the sum of potential increments along a sampled string has finite second moment. Under this condition, the estimator is unbiased (\Cref{prop:unbiasedness}), while its variance depends on the choice of potential. This raises the natural question of when a potential reduces variance, and how much reduction is possible. We answer this question in the next proposition, which identifies a zero-variance potential, which we call the oracle potential, and characterizes exactly when an arbitrary potential improves on the unbiased baseline.

\begin{proposition}[Oracle potential and variance comparison]
\label{prop:oracle-potential}
\label{prop:potential-estimator}
The \defn{oracle potential}, defined on prefixes of positive probability,
satisfies
\[
\ppot^\star(\prefix)
=
\E_{\plm}\!\left[
\testf(\yvar)\mid\yvar\succeq\prefix
\right]
\quad\Longrightarrow\quad
\hest(\ppot^\star)=\estimand,
\qquad
\Var_{\plm}\!\left(\hest(\ppot^\star)\right)=0.
\]
 Let $\hest_*$ be any unbiased
baseline estimator under $\plm$. For every fixed potential $\ppot$, 
\[
\Var_{\plm}\bigl(\hest(\ppot)\bigr)
\leq
\Var_{\plm}(\hest_*)
\iff
\Cov_{\plm}\bigl(\hest_*,\hest_*-\hest(\ppot)\bigr)
\geq
\frac{1}{2}
\Var_{\plm}\bigl(\hest_*-\hest(\ppot)\bigr).
\]
\end{proposition}

Despite the above condition need not hold for a chosen potential, we can exploit the unbiasdness of $\hest_*$ and $\hest(\ppot)$ to build a control-variate estimator with variance equal or lower than $\hest_*$.

\begin{definition}[Potential control-variate estimator]
\label{def:potential-cv}
Given a potential $\ppot$ and an unbiased baseline estimator $\hest_*$,
define
\begin{equation}
\label{eq:potential-cv}
\hest(\ppot,\cvcoef)
\defeq
\hest_*
-
\cvcoef\bigl(\hest_*-\hest(\ppot)\bigr),
\qquad
\cvcoef\in\R.
\end{equation}
\end{definition}

The coefficient $\cvcoef$ controls the contribution of the correction.
Setting $\cvcoef=0$ recovers the baseline estimator, while $\cvcoef=1$
recovers the potential estimator. As is standard for control variates, we choose this coefficient to minimize variance. The next proposition gives the optimal choice.

\begin{proposition}[Unbiasedness and optimal coefficient]
\label{prop:optimal-alpha}
For every fixed potential $\ppot$ and fixed coefficient $\cvcoef$,
the estimator is unbiased:
$\E_{\plm}[\hest(\ppot,\cvcoef)]=\estimand$.
If both $\hest_*$ and $\hest(\ppot)$ have finite second moments and
$\Var_{\plm}\bigl(\hest_*-\hest(\ppot)\bigr)>0$, the
variance-minimizing coefficient satisfies
\begin{equation}
\label{eq:optimal-alpha}
\astar
=
\frac{
\Cov_{\plm}\bigl(\hest_*,\hest_*-\hest(\ppot)\bigr)
}{
\Var_{\plm}\bigl(\hest_*-\hest(\ppot)\bigr)
}
\quad\Longrightarrow\quad
\Var_{\plm}\bigl(\hest(\ppot,\astar)\bigr)
\leq
\Var_{\plm}(\hest_*).
\end{equation}
If the difference has zero variance, the two estimators agree almost
surely, and the choice of coefficient has no effect.
\end{proposition}

Having characterized the role of the potential and the optimal coefficient, we next turn to their practical construction and estimation. In the next section, we describe how to design and select potentials and estimate parameters such as $\alpha^*$ from the same samples, without sacrificing unbiasedness.



\section{Potential Design in practice: applications and experiments}
\label{sec:practice}
In this section, we show how to construct potentials across several estimands and applications, and ask whether the resulting potential estimators achieve lower variance at comparable computational cost. Depending on the structure of $\testf$, we follow one of two complementary approaches: we either exploit known additive structure to derive a potential analytically, or learn a potential from data. We present three applications in the main text and a fourth in \Cref{app:exp-watermark} (\nameref{app:exp-watermark}).
\subsection{Estimating proper losses, divergences, and gradients}
\label{sec:app-losses}

\begin{wraptable}{r}{0.50\linewidth}
\vspace{-\baselineskip}
\setlength{\dtcardextra}{0pt}
\refstepcounter{table}\label{tab:entropy}%
\dtcard{@{}V@{\hspace{5pt}}l@{\hspace{5pt}}V@{}}{%
\toprule
\dtcardhead{Table~\thetable: Entropy} \\
\midrule
$\testf(\yvar)=-\log \plm(\yvar)$\\

\midrule
{$\begin{aligned}
\ppot^\star(\xprefix)&=-\log\pprob(\xprefix)\\
&\quad+\E_{\plm}\!\left[-\log\frac{\plm(\yvar)}{\pprob(\xprefix)}\;\middle|\;\yvar\succeq\xprefix\right]
\end{aligned}$}\\
\midrule
$\ppot(\xprefix)=-\log \pprob(\xprefix)$\\
\midrule
$\Delta\ppot(a\mid\xprefix)=-\log \pprob(a\mid\xprefix)$\\
\midrule
$\Delta\ppot(\eos\mid\xprefix)=-\log \pprob(\eos\mid\xprefix)$\\
\bottomrule}
\end{wraptable}
We estimate seven losses frequently used in language model training and evaluation along with their gradients: entropy \citep{dai2018credit,braverman2020calibration}, KL divergence \citep{schulman2017proximal, ouyang2022training}, Brier loss and power-score regrets \citep{shao2024language} with exponents $\beta \in \{1.25, 1.5, 1.75, 2\}$.

\paragraph{Defining a potential analytically using the function value at the observed prefix.}\label{par:observed-prefix}\mbox{}\\
The above losses and their gradients  have a natural additive structure over prefixes; for any such function, the value of the function at each observed prefix yields a potential. \Cref{tab:entropy} illustrates this for entropy. {For a complete string $\str\succeq\xprefix$, the chain rule splits $-\log\plm(\str)$ into a prefix term $-\log\pprob(\xprefix)$ and a continuation term $-\log(\plm(\str)/\pprob(\xprefix))$, which includes termination.} The oracle potential adds the expected continuation term to the prefix term; keeping only the prefix term thus approximates the oracle potential by its observed component. Here the increments are the per-token surprisals $-\log \pprob(a \mid \xprefix)$ (with a final increment at \eos), which are computed during sampling anyway, so the potential is free. We define the potentials for the other quantities and their gradients in the appendix: \Cref{app:proper-quantities} gives those for KL divergence, the Brier loss and the power-score regrets (\Cref{tab:kl,tab:brier,tab:power}), and \Cref{app:grad-potentials} those for the gradients (\Cref{tab:grad}).


For KL, our potential estimator coincides with that of \citet{amini2025better}; we reproduce their experiment in \Cref{app:exp-kl} (\nameref{app:exp-kl}). We still include KL in our evaluation because their claimed variance ordering relative to Monte Carlo does not hold in general, as shown by \Cref{prop:oracle-potential} and the counterexample in \Cref{app:counterexample}. Across all quantities and gradients, we therefore test how consistently the estimator reduces variance and whether fitting the control-variate coefficient provides additional gains.

\subsubsection{Experiments and results}
\begin{figure}[!htb]
\centering
\begin{minipage}[t]{0.33\textwidth}
\centering
{\footnotesize\textbf{(a)} variance ratio per quantity}\\[3pt]
\includegraphics[width=\linewidth]{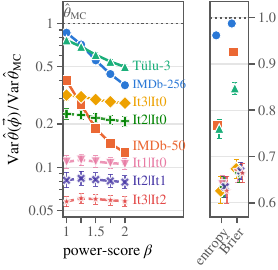}
\end{minipage}\hfill
\begin{minipage}[t]{0.33\textwidth}
\centering
{\footnotesize\textbf{(b)} $\KL$ ratio vs.\ estimated $\KL$}\\[3pt]
\includegraphics[width=\linewidth]{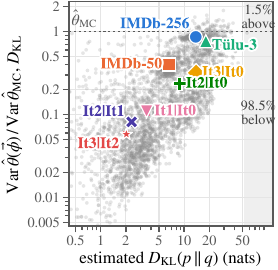}
\end{minipage}\hfill
\begin{minipage}[t]{0.33\textwidth}
\centering
{\footnotesize\textbf{(c)} cross-fitted control variate}\\[3pt]
\includegraphics[width=\linewidth]{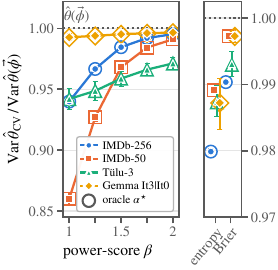}
\end{minipage}
\caption{\footnotesize Variance reduction in the proper-loss experiment (lower is better).
\textbf{(a)} Variance ratio of the potential estimator to Monte Carlo across
power scores ($\beta\to1$ gives $D_{\mathrm{KL}}$); entropy and Brier are shown
on the right. Each series is one policy--reference pair (see \cref{tab:proper-pairs}): solid lines are the four main policy–reference pairs and dashed lines are consecutive Gemma-2 SPPO pairs.
\textbf{(b)} For $D_{\mathrm{KL}}$, variance ratio against the estimated
$D_{\mathrm{KL}}(p\|q)$, with one colored marker per pair and grey markers for
individual prompts from all pairs. The shaded regions indicate prompts below
and above the Monte Carlo variance.
\textbf{(c)} Variance ratio of the cross-fitted control-variate estimator to
the potential estimator; hollow markers use the
full-sample oracle coefficient.
}
\label{fig:exp1-proper}
\end{figure}

\textbf{The potential estimator reduces variance in every configuration,  by up to  $17\times$  in effective sample size.}
We evaluate the estimator on eight configurations from three post-training families (full details in \Cref{app:proper-pairs}, \nameref{app:proper-pairs}). Across all $56$ pair--quantity combinations, the potential reduces variance relative to Monte Carlo (\Cref{fig:exp1-proper}a), with ratios between $0.058$ and $0.987$ and every bootstrap upper bound below one. These correspond to effective sample-size multipliers of $1.01\times$--$17.3\times$.

\textbf{The gains are largest when the policy and reference are close.}
The variation across policy--reference pairs is larger than across estimands. For KL at horizon $256$, the variance ratio is $0.86$ on GPT-2 IMDb, $0.77$ on T\"ulu-3, and $0.32$ on Gemma-2, dropping to $0.06$--$0.11$ between consecutive Gemma SPPO iterations (equivalent to $9.1\times$--$16.7\times$ the Monte Carlo sample size). The same pattern appears across prompts: smaller KL is associated with larger variance reductions (\Cref{fig:exp1-proper}b).

\textbf{Variance reduction holds for the vast majority of prompts, but not all of them}: Monte Carlo has lower variance on $1.5\%$ of all $4{,}096$ prompts, including $12.3\%$ of T\"ulu-3 prompts (95\% interval $9.6\%$--$15.2\%$). This provides empirical evidence that the general KL variance ordering of \citet{amini2025better} does not hold universally; Fitting the control-variate coefficient yields ratios of $0.86$--$1.00$ relative to the untuned potential estimator (\Cref{fig:exp1-proper}c); among KL prompts where the potential estimator increases variance, fitting $\alpha$ lowers the median variance ratio relative to Monte Carlo from $1.134$ to $0.974$. \looseness=-1

\Cref{app:exp-gradients} repeats the comparison for the gradient of each estimand with respect to the policy parameters along a DPO trajectory of the GPT-2 sentiment policy: early in training, the potential removes most of the gradient variance; when it helps little on its own, fitting the control-variate coefficient recovers substantial gains.

\subsection{Estimating Expected Counts and Frequencies of n-grams}

We next study expected counts and frequencies of word $n$-grams, i.e., contiguous
sequences of $n$ words. These statistics arise in applications such as model
distillation \citep{suresh2021approximating,krishnan2023n} and computational
psycholinguistics \citep{opedal-etal-2024-role,kiegeland2026proper}. More broadly, $n$-gram counts are one instance of a general class of
functions $\testf(\yvar)=\#\{\text{occurrences of substring }\boldsymbol{w}\text{ in }\yvar\}$, encompassing quantities such
as the number of forbidden keywords in a string, unsafe API calls, or other keyword-based
patterns emitted by a model.

\paragraph{Designing potentials that marginalize over multiple future tokens.} 
We start by observing that, for any target substring
$\boldsymbol{w}\in\Sigma^\ast$, 
the count function $\testf$ admits two natural decompositions:\[
\testf(\yvar)
=
\sum_t Z^{\mathrm{end}}_{\boldsymbol{w},t}
=
\sum_t Z^{\mathrm{start}}_{\boldsymbol{w},t},
\qquad
Z^{\mathrm{end}}_{\boldsymbol{w},t}
=
\1\{\boldsymbol{w}\text{ ends at }t\},
\qquad
Z^{\mathrm{start}}_{\boldsymbol{w},t}
=
\1\{\boldsymbol{w}\text{ starts at }t\}.
\]
\begin{wraptable}{r}{0.66\linewidth}
\vspace{-0.6\baselineskip}
\setlength{\dtcardextra}{0pt}
\refstepcounter{table}\label{tab:ngram}%
\dtcard{@{}V@{\hspace{5pt}}l@{\hspace{5pt}}V@{}}{%
\toprule
\dtcardhead{Table~\thetable: Substring counts, lookahead occurrence potential} \\
\midrule
$\testf(\yvar)=\#\{\text{occurrences of }\boldsymbol{w}\text{ in }\yvar\}$\\
\midrule
$\testf(\yvar)=\sum_{t}Z^{\mathrm{start}}_{\boldsymbol{w},t}$, \ $Z^{\mathrm{start}}_{\boldsymbol{w},t}=Z^{\mathrm{can}}_{\boldsymbol{w},t}+Z^{\mathrm{noncan}}_{\boldsymbol{w},t}$\\
\midrule
{$\ppot^\star(\xprefix)=\E_{\plm}\!\left[\sum_{t=1}^{|\yvar|}Z^{\mathrm{start}}_{\boldsymbol{w},t}\;\middle|\;\yvar\succeq\xprefix\right]$}\\
\midrule
$\ppot(\xprefix)=\sum_{t\le|\xprefix|}\E_{\plm}\!\left[Z^{\mathrm{can}}_{\boldsymbol{w},t}\mid \yvar_{\le t}=\xprefix_{\le t}\right]$\\
\midrule
$\Delta\ppot(a\mid\xprefix)=\E_{\plm}\!\left[Z^{\mathrm{can}}_{\boldsymbol{w},t}\mid \yvar_{\le t}=\xprefix a\right]$, \ $t=|\xprefix|+1$\\
\midrule
$\Delta\ppot(\eos\mid\xprefix)=\sum_{t\le|\xprefix|}\bigl(Z^{\mathrm{can}}_{\boldsymbol{w},t}-\E_{\plm}[Z^{\mathrm{can}}_{\boldsymbol{w},t}\mid\yvar_{\le t}]\bigr)+\sum_{t}Z^{\mathrm{noncan}}_{\boldsymbol{w},t}$\\
\bottomrule}
\end{wraptable}

The $Z^{\mathrm{end}}_{\boldsymbol{w}}$ decomposition yields a potential that is completely determined by evaluating $f$ at prefix $x$.  Its expected increment, however, is nonzero only at prefixes that already match the substring up to its final token, so marginalization is useful only at those few prefixes.

By contrast, $Z^{\mathrm{start}}_{\boldsymbol{w}}$ is generally still unobserved at the current prefix. We therefore replace it by its full conditional expectation---the probability that the occurrence starting at that position is eventually completed---and take the cumulative expected count over starting positions as the potential.
In practice, we marginalize over the canonical tokenization $\boldsymbol{w}=(w_1,\ldots,w_L)$ with $w_i\in\Sigma$ and $L \defeq |\boldsymbol{w}|$. At each sampled prefix, this requires $L$ additional token evaluations: the remaining target tokens after the next token, plus the word boundary. Thus, if $C$ denotes the cost of generating one sample from model $\plm$, the additional marginalization cost is $O(LC)$. Including generation gives a total per-target token-evaluation cost proxy of $(L+1)C$. \looseness=-1

\Cref{tab:ngram} gives the mathematical definition of the potential using $Z^{\mathrm{start}}_{\boldsymbol{w}}$ and its increments. This example illustrates the flexibility of the telescoping construction in \Cref{eq:additive-decomposition}: the potential may marginalize only a selected component of $\testf$---here, occurrences under the canonical tokenization---while the terminal increment recovers the remainder, so the decomposition still equals $\testf$ exactly at termination.\footnote{In this application, part of the terminal increment has expectation zero under $\plm$ and can therefore be dropped without changing the estimator's expectation, yielding $\widetilde{\Delta}\ppot(\eos\mid\xprefix)=\sum_t Z^{\mathrm{noncan}}_{\boldsymbol{w},t}$.}

\subsubsection{Experiments \& Results}

\textbf{Setup.}
We construct $7{,}647$ word $n$-gram targets ($n=1,\ldots,6$) from the English texts of MECO \citep{siegelman2022expanding}, with canonical tokenization lengths ranging from $1$ to $10$ tokens. We ask whether the potential estimator that uses the potential defined in \cref{tab:ngram} reduces variance at comparable computational cost with MC, and how the gains vary with target rarity;  

\begin{figure}[!t]
\centering
\begin{minipage}[t]{0.33\textwidth}
\centering
{\footnotesize\textbf{(a)} $\widehat{\rho}_c(w)$ with log-ratio 95\% CI}\\[3pt]
\includegraphics[width=\linewidth]{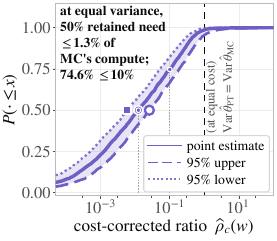}
\end{minipage}\hfill
\begin{minipage}[t]{0.33\textwidth}
\centering
{\footnotesize\textbf{(b)} $\widehat{\rho}_c(w)$ against target rarity}\\[3pt]
\includegraphics[width=\linewidth]{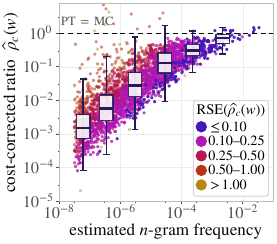}
\end{minipage}\hfill
\begin{minipage}[t]{0.33\textwidth}
\centering
{\footnotesize\textbf{(c)} resolved by token length}\\[3pt]
\includegraphics[width=\linewidth]{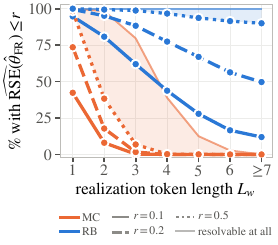}
\end{minipage}
\caption{\footnotesize {Cost-adjusted variance ratio $\widehat{\rho}_c(w)$ for $n$-gram frequency estimation. \textbf{(a)} Empirical cumulative distribution over $3{,}700$ targets with bootstrap RSE at most $2$, with approximate $95\%$ log-ratio confidence bounds. \textbf{(b)} Variance reduction versus estimated frequency for the $2{,}831$ targets with at least $25$ expected occurrences in the independent frequency bank. \textbf{(c)} Fraction of targets whose frequency estimate reaches each RSE threshold, grouped by token length, at the same sample size $M=4{,}000$ for both estimators. Dashed parity lines in (a--b) mark equal variance at equal cost under the proxy; in (c), line styles distinguish RSE thresholds.}}
\label{fig:headline}
\end{figure}

For each target, we estimate its expected count and frequency (the ratio estimator of \Cref{eq:frequency-estimators}, \Cref{app:ngram-freq}) from \(4{,}000\) GPT-2 samples using the potential estimator. We estimate the Monte Carlo variance from an independent \(4\times10^6\)-sample bank drawn from the same model. We then report the cost-adjusted variance ratio \(\widehat{\rho}_c(w)\) (\Cref{eq:variance-ratio}, \Cref{app:ngram-rho}), scaling the potential-estimator variance by the per-target cost proxy \(L_w+1\) and the Monte Carlo variance to \(4{,}000\) samples. Thus, \(\widehat{\rho}_c(w)<1\) indicates lower variance at equal cost under this proxy.

To measure target rarity independently, we estimate each target's frequency from a separate \(2\times10^6\)-sample bank. We report frequency results here because they are easier to interpret, and defer count results to \Cref{app:exp-ngram}; both show the largest gains for rare targets. The appendix also reports results for the simpler observed-prefix potential (\Cref{par:observed-prefix}), which provides virtually no variance reduction.
\paragraph{{Half the retained targets reach equal variance on at most \(1.3\%\) of the Monte Carlo compute.}}
{Panel~\Cref{fig:headline}a reports the cost-adjusted variance ratio \(\widehat{\rho}_c\) for the \(3{,}700\) of \(7{,}647\) targets estimated with relative standard error at most \(2\), with bootstrap \(95\%\) confidence intervals. The median ratio is \(1.24\%\), and \(74.7\%\) of targets require at most \(10\%\) of the Monte Carlo compute. Using the upper confidence bound instead gives a median of \(2.56\%\), with \(95.6\%\) of targets still below parity. This shows that the efficiency gains remain substantial even under a conservative treatment of uncertainty.}


\textbf{At the same sample size ($M=4{,}000$), the potential estimator yields substantially more precise frequency estimates}, with $79.7\%$ of targets reaching RSE $\leq 20\%$ and $96.5\%$ reaching RSE $\leq 50\%$, compared with $11.4\%$ and $19.0\%$ for Monte Carlo (see \Cref{fig:headline}c). The gap grows with token length: for targets of length $4$ or more, Monte Carlo almost never reaches RSE $\leq 50\%$, whereas the potential estimator does so for at least $90\%$ of targets in every length group. For $L_w\geq 7$, Monte Carlo does not observe any target even in the $4\times10^6$-sample reference bank, while the potential estimator gives a finite estimate for all of them.


\subsection{Event analysis with language models}
\label{sec:app-events}
\leavevmode\citet{dorman2026rare} introduce an importance-sampling framework for estimating probabilities of events defined by a real-valued function $\testf$ over strings. They partition the range of $\testf$ into $J$ intervals $\{D_j\}_{j=1}^J$, sample from distributions tilted toward different values of $\testf$, and reweight the resulting trajectories to estimate $\Pr_{\plm}[\testf(\yvar)\in D_j]$, for each $j$ . We augment their estimator with potentials for the corresponding event indicators $\testf_j(\yvar)=\1\{\testf(\yvar)\in D_j\}$.

\paragraph{Learning potentials from samples.}
We note that these indicators need not admit a useful additive decomposition over prefixes. A simple baseline is the \emph{terminal potential}, which assigns zero to unfinished prefixes and the indicator value $\testf_j(\yvar)$ at the string termination. To exploit information from earlier prefixes, we instead learn an approximation to the oracle potential: the conditional probability that $\testf_j(\yvar)=1$ given the observed prefix. Each sampled prefix $\xprefix_i$, $i=1,\ldots,B$, is paired with its trajectory's outcome $\testf_j(\str_i)$, $\str_i\succeq\xprefix_i$. Using prefix features $\boldsymbol{\xi}(\xprefix)\in\mathbb{R}^d$, we fit a predictor\footnote{When trajectories are obtained through importance sampling as in this application, we weight each prefix's prediction loss by its trajectory's importance weight, so that fitting targets conditional event probabilities under the original language model.} $g_j(\cdot;\boldsymbol{\eta})\colon\mathbb{R}^d\rightarrow[0,1]$ with parameters $\boldsymbol{\eta}\in\mathbb{R}^{m}$. We can define the potential \timv{Notation consistence: use $\Pr[\cdot]$ instead of $\Pr(\cdot)$ \response{francesco} Done}
\begin{equation}
\label{eq:learned-potential}
\ppot_j(\xprefix;\boldsymbol{\eta})
=
g_j(\boldsymbol{\xi}(\xprefix);\boldsymbol{\eta})
\approx
\ppot_j^\star(\xprefix)
=
\Pr_{\plm}[\testf_j(\yvar)=1\mid\yvar\succeq\xprefix].
\end{equation}

\paragraph{Unbiased estimation with learned potentials.}
Fitting and evaluating potentials on the same samples can introduce bias. We avoid this through $K$-fold cross-fitting \citep{chernozhukov2018double}: each potential is trained on the other folds and evaluated on the held-out fold. Averaging the fold-specific estimates preserves unbiasedness under independent sampling and exact importance weights (when needed), while using all samples.

\paragraph{Potential selection with nested cross-fitting.}
Different potentials may achieve different variance reductions, motivating data-driven selection. However, selecting a potential on the same samples used for estimation can introduce bias \citep{varma2006bias}. We therefore use nested cross-fitting \citep{cawley2010over}: for each held-out fold, we select the potential through inner cross-validation on the remaining data, then apply it to the held-out fold. We detail this procedure in \Cref{alg:potential-selection}.
\subsubsection{Experiments \& Results}

\paragraph{Setup.}
We replicate the main experiment of \citet{dorman2026rare}, estimating the distribution of the Automated Readability Index (ARI; \Cref{eq:ari}), a sequence-level readability score, over texts generated by TinyStories-8M. Here $\testf=\mathrm{ARI}$. During sampling, we additionally retain the next-token distributions for subsequent potential fitting and evaluation. We consider three candidate potentials: the terminal potential, which marginalizes over the final token, and two learned potentials based on a simple logistic regression and a mixture model, respectively. We defer their details to \Cref{app:learned-potentials} (\nameref{app:learned-potentials}). We also include a control-variate variant of the terminal potential with a learned coefficient $\alpha$, giving four candidates, among which we select using nested cross-fitting with $K=5$ outer folds. Throughout, we refer to the original estimator of \citet{dorman2026rare} by the acronym MBAR. {MBAR uses estimated weights and is asymptotically unbiased under its regularity assumptions \citep{shirts2008statistically}; we do not claim exact finite-sample unbiasedness for this experiment.}

\paragraph{Learned potentials reduce variance by ${\sim}12\%$, equivalent to
${\sim}0.9\times10^{6}$ additional MBAR samples.}\mbox{}\par

\begin{wraptable}{r}{0.44\textwidth}
\vspace{-5pt}
\centering\small\setlength{\tabcolsep}{4pt}
\begin{tabular}{l c}
\toprule
Estimator & Var.\ ratio vs.\ MBAR [95\% CI] \\
\midrule
Fitted selection & 0.88 [0.86, 0.91] \\
Terminal & 0.92 [0.90, 0.94] \\
Terminal, $\alpha$ & 0.92 [0.90, 0.95] \\
Logistic & 0.90 [0.86, 0.91] \\
Mixture model & 0.87 [0.84, 0.90] \\
\bottomrule
\end{tabular}
\caption{variance relative to MBAR, median over the $64$ populated ARI bins; $95\%$ intervals from $1{,}000$ paired bootstrap replicas. Full version in \Cref{tab:rea-nested-full}.}

\label{tab:rea-nested}
\vspace{-8pt}
\end{wraptable}
\Cref{tab:rea-nested} summarizes the variance ratio of each potential estimator relative to MBAR; lower values indicate a more efficient estimator. The learned potentials give the largest reductions: the logistic and mixture-model potentials attain median variance ratios of $0.90$ and $0.87$, respectively, and nested selection attains $0.88$, closely matching the best individual candidate. The terminal potential gives a smaller but still consistent reduction, with a median ratio of $0.92$. It exploits cases where membership in a bin $D_j$ is determined only at the final token.

The experiment uses $N=6.68\times10^{6}$ generations, matching the original experiments. Assuming variance scales inversely with sample size, the median ratio of $0.88$ corresponds to an MBAR budget requiring $0.9\times10^{6}$ more generations than our potential estimator. We used only CPUs for potential estimation, as detailed in \Cref{app:events-results} (\nameref{app:events-results}).
\section{Conclusion}
\label{sec:conclusion}
We introduced a general framework for reducing the variance of expectation estimates under autoregressive language models using next-token distributions available as a by-product of sampling. By repurposing potentials from reinforcement learning, we construct an unbiased estimator whose variance depends on the chosen potential. Across several estimands and applications, we showed how practical potentials can improve statistical efficiency at comparable computational cost. More broadly, our results establish potential design as a general strategy for constructing lower-variance estimators under language models.

\subsection*{AI use statement}
In this work, we used generative AI tools to provide feedback on research methodology and experimental design, assist with the interpretation of experimental results, and help check and refine mathematical statements and derivations. Additionally, we used generative AI tools to assist with software development, and drafting and editing portions of the manuscript for clarity, structure, and readability. AI-assisted text and code were reviewed and verified by the authors.  We take responsibility for the final content of this work, including text, claims, and artifacts produced with the aid of generative AI.

\bibliography{references,anthology-1,anthology-2} 
\bibliographystyle{iclr2027_conference}

\appendix
\crefalias{section}{appendix}
\crefalias{subsection}{appendix}    
\crefalias{subsubsection}{appendix}

\clearpage
\etocsettocdepth.toc{subsubsection}
\begingroup
\etocsettocstyle{\subsection*{\centering Appendix Contents}}{}
\tableofcontents
\endgroup
\bigskip

\newpage
\section*{Limitations}
Our framework has several limitations. First, variance reduction depends on the choice of potential. Although the oracle potential yields zero variance under direct sampling, practical potentials only approximate it and are not guaranteed to improve over the baseline for every individual prompt. The control-variate formulation provides a safeguard against poorly chosen potentials, although estimating its coefficient or selecting among learned potentials requires cross-fitting to preserve unbiasedness. Thus, the strong performance of the potentials studied here should not be interpreted as implying that arbitrary potentials will yield similar gains.

Second, the computational advantage depends on the cost of evaluating the potential. In several of our applications, the required quantities are already available during sampling or can be computed relatively cheaply, whereas more elaborate potentials may require additional model evaluations or fitting. At the same time, this tradeoff suggests a natural direction for future work: designing potentials that achieve stronger variance reduction while remaining computationally inexpensive.

Third, our approach requires access to the model's next-token probability distribution along sampled prefixes. It therefore applies directly when logits or token probabilities are available, but not to black-box generation APIs that expose only sampled outputs. Finally, our empirical conclusions are necessarily limited to the estimands, models, and experimental configurations considered here; the magnitude of the variance reductions may differ in other settings.
\section{Related Work}
\label{app:related}

\paragraph{Importance sampling and sequential Monte Carlo for language models.}
Several methods reduce variance by modifying the sampling distribution to favor events of interest and correcting through importance weighting \citep{angell2026estimating,dorman2026rare}.
For example, \citet{boyd2022predictive} combine query-specific importance sampling with beam search to estimate probabilities of sequence events.
Sequential Monte Carlo \citep{delmoral2004feynman} and its twisted variants \citep{whiteley2014twisted} have also been applied to controlled generation \citep{lew2023sequential,zhao2024probabilistic,loula2025syntactic} and rare-event estimation \citep{liu2026adaptive}. In such methods, a \emph{twist} assigns a nonnegative weight to each prefix, guiding sampling and resampling toward prefixes judged more likely to produce the desired outcome.
For event-probability estimation, the optimal twist weights each prefix according to the probability that completing it under the original language model will satisfy the event.
It is therefore proportional to our oracle potential $\ppot^\star$, which represents the same conditional probability \citep{zhao2024probabilistic}.
Our framework uses prefix potentials to construct additive, zero-mean corrections that enable variance reduction after sampling, exploiting the next-token distribution created at generation.
We demonstrate its complementarity with importance sampling by augmenting the importance sampling estimator of \citet{dorman2026rare}, who sample from exponentially tilted sequence distributions using Markov chain Monte Carlo.

\paragraph{Conditional Monte Carlo, Rao--Blackwellization, and control variates.}
Replacing a finite-variance estimator by its conditional expectation cannot increase its variance, the principle underlying conditional Monte Carlo \citep{hammersley1956conditional} and Rao--Blackwellization \citep{robert2004monte,owen2013monte}.
For KL estimation under language models, \citet{amini2025better} replace each sampled token log-likelihood ratio by the exact KL divergence between the next-token distributions at the sampled prefix, and derive an analogous gradient estimator.
Although their KL estimator is unbiased, the variance-reduction guarantee in their Theorem~2 does not hold in general: averaging individual token contributions can alter their covariances and increase the variance of the sum.
Our framework recovers their KL estimator as a particular choice of potential and extends the construction to arbitrary test functions, explicitly distinguishing unbiasedness from variance reduction.
The difference between a potential estimator and any unbiased baseline defines a zero-mean control variate \citep{owen2013monte}.
This correction preserves unbiasedness for any fixed coefficient, while the population-optimal coefficient yields variance no greater than that of the baseline.
\paragraph{Potentials and reward shaping in reinforcement learning.}
Potential-based reward shaping in reinforcement learning uses changes in a function of the state to redistribute rewards while preserving optimal policies \citep{ng1999policy,wiewiora2003potential}.
Value functions, which predict future rewards, provide natural potentials: the resulting shaped rewards measure prediction errors and are used in generalized advantage estimation \citep{schulman2015high}.
Our construction uses the same telescoping principle, with text prefixes as states and the test function's value as the terminal reward.
We then average the potential increments over the next-token distribution to construct unbiased expectation estimates after sampling.
The ideal potential predicts the expected terminal reward given each prefix; its averaged increments vanish, yielding our zero-variance oracle (\Cref{prop:oracle-potential}).

\section{Additional Propositions and Proofs}
\label{app:proofs}

Throughout, $\yvar\sim\plm$ denotes a single sampled string, $|\yvar|$ its length excluding $\eos$, and $\yvar_{<t}$ its prefix of length $t-1$, so that $\yvar_{<1}=\varepsilon$ and $\yvar_{<|\yvar|+1}=\yvar$. All expectations are assumed finite.

\subsection{Importance Sampling and the Optimal Proposal}
\label{app:is-variance}

\begin{proposition}[Importance-sampling variance and optimal proposal]
\label{prop:is-variance}
Let $\qlm$ be a proposal {satisfying $\qlm(\str)>0$ whenever $\plm(\str)|\testf(\str)|>0$}, and let $\thIS$ be the importance-sampling estimator of \cref{eq:is-estimator}. Then $\E_{\qlm}[\thIS]=\estimand$ and
\begin{equation}
\label{eq:is-variance-general}
\Var_{\qlm}(\thIS)
=
\frac1M\left(
\sum_{\str\in\kleene{\alphabet}}
\frac{\plm(\str)^2\testf(\str)^2}{\qlm(\str)}
-
\estimand^2
\right).
\end{equation}
{Terms with $\qlm(\str)=0$ are read as zero.}
If $\E_{\plm}|\testf(\yvar)|>0$, the variance is minimized uniquely by
\begin{equation}
\label{eq:is-optimal-proposal}
\qlm^\star(\str)
=
\frac{\plm(\str)\,|\testf(\str)|}{\E_{\plm}|\testf(\yvar)|},
\qquad
\Var_{\qlm^\star}(\thIS)
=
\frac1M\Bigl(\bigl(\E_{\plm}|\testf(\yvar)|\bigr)^2-\estimand^2\Bigr),
\end{equation}
and the minimum is zero if and only if $\testf$ has constant sign on the support of $\plm$.
\end{proposition}

\begin{proof}
The summands of $\thIS$ are i.i.d., so it suffices to treat $M=1$. {Using the support assumption,}
\begin{align*}
\E_{\qlm}\!\left[\frac{\plm(\yvar)}{\qlm(\yvar)}\testf(\yvar)\right]
&=\sum_{\str:\,\qlm(\str)>0}\qlm(\str)\frac{\plm(\str)}{\qlm(\str)}\testf(\str)
=\sum_{\str\in\kleene{\alphabet}}\plm(\str)\testf(\str)
=\estimand,
\end{align*}
where the middle equality uses that {$\plm(\str)\testf(\str)=0$ whenever $\qlm(\str)=0$}. The same computation for the second moment gives $\E_{\qlm}[(\plm(\yvar)\testf(\yvar)/\qlm(\yvar))^2]=\sum_{\str}\plm(\str)^2\testf(\str)^2/\qlm(\str)$, which yields \cref{eq:is-variance-general}.

For the optimal proposal, write $c\defeq\E_{\plm}|\testf(\yvar)|=\sum_{\str}\plm(\str)|\testf(\str)|$. By the Cauchy--Schwarz inequality applied to the vectors $\bigl(\plm(\str)|\testf(\str)|/\sqrt{\qlm(\str)}\bigr)_{\str}$ and $\bigl(\sqrt{\qlm(\str)}\bigr)_{\str}$ over the support of $\qlm$,
\begin{align*}
c^2
=\Bigl(\sum_{\str}\frac{\plm(\str)|\testf(\str)|}{\sqrt{\qlm(\str)}}\sqrt{\qlm(\str)}\Bigr)^2
\le
\sum_{\str}\frac{\plm(\str)^2\testf(\str)^2}{\qlm(\str)}
\sum_{\str}\qlm(\str)
=
\sum_{\str}\frac{\plm(\str)^2\testf(\str)^2}{\qlm(\str)}.
\end{align*}
Hence $\Var_{\qlm}(\thIS)\ge(c^2-\estimand^2)/M$ for every admissible $\qlm$, with equality if and only if the two vectors are proportional, that is, $\qlm(\str)=\plm(\str)|\testf(\str)|/c=\qlm^\star(\str)$ on the support of $\qlm$; since $\qlm^\star$ sums to one, equality forces $\qlm=\qlm^\star$. Substituting $\qlm^\star$ into \cref{eq:is-variance-general} confirms the value in \cref{eq:is-optimal-proposal}. Finally, $c^2-\estimand^2=(c+|\estimand|)(c-|\estimand|)\ge0$ because $c=\E_{\plm}|\testf(\yvar)|\ge|\E_{\plm}\testf(\yvar)|=|\estimand|$, with equality if and only if $\testf(\yvar)$ has almost surely constant sign under $\plm$.
\end{proof}

\subsection{Unbiasedness of the Potential Estimator}
\label{app:unbiasedness}

\begin{proposition}[Unbiasedness and Variance]
\label{prop:unbiasedness}
Let $\ppot$ be any potential whose increments are absolutely summable in expectation, $\E_{\plm}\bigl[\sum_{t=1}^{|\yvar|+1}|\Delta\ppot(\strv_t\mid\yvar_{<t})|\bigr]<\infty$. Then the potential estimator of \cref{eq:potential-estimator} is unbiased: $\E_{\plm}[\hest(\ppot)]=\estimand$.
\end{proposition}

\begin{proof}
It suffices to treat $M=1$.
Write $\horizon\defeq|\yvar|+1$ for the index of the $\eos$ step, so that the event $\{t\le\horizon\}$ is determined by $\yvar_{<t}$: it holds exactly when $\yvar_{<t}$ has not yet been terminated. For each $t$, conditioning on $\yvar_{<t}$ and using that $\strv_t\sim\pprob(\cdot\mid\yvar_{<t})$ on this event,
\[
\E_{\plm}\bigl[\1\{t\le\horizon\}\,\Delta\ppot(\strv_t\mid\yvar_{<t})\bigr]
=
\E_{\plm}\bigl[\1\{t\le\horizon\}\,\E[\Delta\ppot(\strv_t\mid\yvar_{<t})\mid\yvar_{<t}]\bigr]
=
\E_{\plm}\bigl[\1\{t\le\horizon\}\,\mathbf B\Delta\ppot(\yvar_{<t})\bigr],
\]
by \cref{eq:bellman-averaged-increment}. Absolute summability lets us sum over $t$ inside the expectation on both sides, giving
\[
\E_{\plm}\Bigl[\sum_{t=1}^{\horizon}\Delta\ppot(\strv_t\mid\yvar_{<t})\Bigr]
=
\E_{\plm}\Bigl[\sum_{t=1}^{\horizon}\mathbf B\Delta\ppot(\yvar_{<t})\Bigr]
=
\E_{\plm}[\hest(\ppot)]-\ppot(\varepsilon).
\]
The left-hand side equals $\E_{\plm}[\testf(\yvar)]-\ppot(\varepsilon)=\estimand-\ppot(\varepsilon)$ by the telescoping identity \cref{eq:additive-decomposition}, which proves the claim.

\end{proof}

\subsection{Proof of \Cref{prop:oracle-potential}}
\label{app:oracle-potential}

\begin{proof}[Proof of \Cref{prop:oracle-potential}]
\emph{Oracle potential.}
For a prefix $\prefix$ with $\pprob(\prefix)>0$, the oracle potential is the continuation value
\[
\ppot^\star(\prefix)
=
\E_{\plm}[\testf(\yvar)\mid\yvar\succeq\prefix]
=
\frac{1}{\pprob(\prefix)}\sum_{\cont\in\kleene{\alphabet}}\plm(\prefix\cont)\,\testf(\prefix\cont),
\]
which is finite whenever $\E_{\plm}|\testf(\yvar)|<\infty$. In particular $\ppot^\star(\varepsilon)=\estimand$. Splitting the continuation $\cont$ according to its first symbol gives the one-step recursion
\begin{equation}
\label{eq:continuation-recursion}
\ppot^\star(\prefix)
=
\sum_{a\in\alphabet}\pprob(a\mid\prefix)\,\ppot^\star(\prefix a)
+
\pprob(\eos\mid\prefix)\,\testf(\prefix),
\end{equation}
where terms with $\pprob(a\mid\prefix)=0$ are read as zero. The right-hand side is precisely $\mathbf B\ppot^\star(\prefix)$, so the expected increment of \cref{eq:bellman-averaged-increment} vanishes at every prefix of positive probability:
\[
\mathbf B\Delta\ppot^\star(\prefix)
=
\mathbf B\ppot^\star(\prefix)-\ppot^\star(\prefix)
=0.
\]
Every prefix $\yvar^{(m)}_{<t}$ visited by a sampled string has positive probability, so all inner sums in \cref{eq:potential-estimator} vanish and
\[
\hest(\ppot^\star)
=
\ppot^\star(\varepsilon)
+
\frac1M\sum_{m=1}^{M}\sum_{t=1}^{|\yvar^{(m)}|+1}0
=
\estimand
\]
pathwise. A constant has zero variance, which proves the first claim.

\emph{Variance comparison.}
Writing $\hest(\ppot)=\hest_*-\bigl(\hest_*-\hest(\ppot)\bigr)$ and expanding,
\[
\Var_{\plm}\bigl(\hest(\ppot)\bigr)
=
\Var_{\plm}(\hest_*)
+
\Var_{\plm}\bigl(\hest_*-\hest(\ppot)\bigr)
-
2\Cov_{\plm}\bigl(\hest_*,\hest_*-\hest(\ppot)\bigr).
\]
Therefore $\Var_{\plm}(\hest(\ppot))\le\Var_{\plm}(\hest_*)$ holds if and only if
$\Var_{\plm}(\hest_*-\hest(\ppot))\le2\Cov_{\plm}(\hest_*,\hest_*-\hest(\ppot))$, which is the stated condition.
\end{proof}

\subsection{Proof of \Cref{prop:optimal-alpha}}
\label{app:optimal-alpha}

\begin{proof}[Proof of \Cref{prop:optimal-alpha}]
\emph{Unbiasedness.}
By linearity,
\[
\E_{\plm}[\hest(\ppot,\cvcoef)]=\E_{\plm}[\hest_*]-\cvcoef\bigl(\E_{\plm}[\hest_*]-\E_{\plm}[\hest(\ppot)]\bigr)=\estimand-\cvcoef(\estimand-\estimand)=\estimand,
\]
using that $\hest_*$ is unbiased by assumption and $\hest(\ppot)$ is unbiased by \Cref{prop:unbiasedness}.

\emph{Optimal coefficient.}
Write $D\defeq\hest_*-\hest(\ppot)$. Since both estimators have finite second moments,
\begin{equation}
\label{eq:quad-cv}
\Var_{\plm}\bigl(\hest(\ppot,\cvcoef)\bigr)
=
\Var_{\plm}(\hest_*-\cvcoef D)
=
\Var_{\plm}(\hest_*)
-
2\cvcoef\,\Cov_{\plm}(\hest_*,D)
+
\cvcoef^2\,\Var_{\plm}(D).
\end{equation}
When $\Var_{\plm}(D)>0$, this is a strictly convex quadratic in $\cvcoef$, minimized where its derivative vanishes, $\cvcoef\,\Var_{\plm}(D)=\Cov_{\plm}(\hest_*,D)$, that is, at $\astar=\Cov_{\plm}(\hest_*,D)/\Var_{\plm}(D)$ as in \cref{eq:optimal-alpha}. Substituting $\astar$ into \cref{eq:quad-cv},
\[
\Var_{\plm}\bigl(\hest(\ppot,\astar)\bigr)
=
\Var_{\plm}(\hest_*)
-
\frac{\Cov_{\plm}(\hest_*,D)^2}{\Var_{\plm}(D)}
\le
\Var_{\plm}(\hest_*).
\]
When $\Var_{\plm}(D)=0$, $D$ is almost surely equal to its mean, which is zero by unbiasedness, so $\hest(\ppot,\cvcoef)=\hest_*$ almost surely for every $\cvcoef$.
\end{proof}

\subsection{Potentials under a Proposal}
\label{app:importance-sampling}

We now drop the assumption $\qlm=\plm$ of \Cref{sec:sequential-decompositions}. Let $\qlm$ be a proposal with respect to which $\plm$ is absolutely continuous, {so that $\qlm(\str)=0$ implies $\plm(\str)=0$}, and let $\yvar^{(1)},\ldots,\yvar^{(M)}\iid\qlm$. The potential construction remains unchanged: given a potential $\ppot$, we first form the corresponding single-sample potential estimator under the target distribution $\plm$, and then apply the usual importance weight to account for sampling from $\qlm$.

\begin{definition}[Importance-weighted potential estimator]
\label{def:is-potential-estimator-mod}
Let $\yvar^{(1)},\ldots,\yvar^{(M)}\iid\qlm$. Given a potential $\ppot$, define
\begin{align}
\label{eq:is-potential-estimator-mod}
\thIS(\ppot)
\defeq
\frac{1}{M}
\sum_{m=1}^{M}
\frac{\plm(\yvar^{(m)})}{\qlm(\yvar^{(m)})}
\left[
\ppot(\varepsilon)
+
\sum_{t=1}^{|\yvar^{(m)}|+1}
\mathbf B\Delta\ppot(\yvar^{(m)}_{<t})
\right].
\end{align}
\end{definition}

 The next proposition is the counterpart of \Cref{prop:oracle-potential} under a proposal.

\begin{proposition}[Unbiasedness and variance under a proposal]
\label{prop:is-potential-variance-mod}
Let $\ppot$ satisfy the summability condition of
\Cref{prop:unbiasedness}, and assume that $\thIS(\ppot)$ has a finite
second moment. Assume that $\thIS(\ppot)$ has a finite second moment. Then
$\E_{\qlm}[\thIS(\ppot)]=\estimand$. Moreover, for the oracle potential,
\[
\Var_{\qlm}\bigl(\thIS(\ppot^\star)\bigr)
=
\frac{\estimand^2}{M}
\left(
\sum_{\str\in\kleene{\alphabet}}
\frac{\plm(\str)^2}{\qlm(\str)}
-1
\right).
\]
In particular, if $\estimand\neq0$, the oracle potential yields zero variance if and only if $\qlm=\plm$.
\end{proposition}

\begin{proof}
\emph{Unbiasedness.} It suffices to treat $M=1$. By a change of measure,
\begin{align*}
\E_{\qlm}[\thIS(\ppot)]
&=
\E_{\qlm}\left[
\frac{\plm(\yvar)}{\qlm(\yvar)}
\left(
\ppot(\varepsilon)
+
\sum_{t=1}^{|\yvar|+1}
\mathbf B\Delta\ppot(\yvar_{<t})
\right)
\right]
\\
&=
\E_{\plm}\left[
\ppot(\varepsilon)
+
\sum_{t=1}^{|\yvar|+1}
\mathbf B\Delta\ppot(\yvar_{<t})
\right]
=
\estimand,
\end{align*}
where the last equality follows from \Cref{prop:unbiasedness}.

\emph{Variance.} By \Cref{prop:oracle-potential},
\[
\ppot^\star(\varepsilon)
+
\sum_{t=1}^{|\yvar|+1}
\mathbf B\Delta\ppot^\star(\yvar_{<t})
=
\estimand
\]
almost surely under $\plm$. Hence,
\begin{align*}
M\Var_{\qlm}\bigl(\thIS(\ppot^\star)\bigr)
&=
\Var_{\qlm}\left(
\frac{\plm(\yvar)}{\qlm(\yvar)}\estimand
\right)
\\
&=
\estimand^2
\left(
\E_{\qlm}\left[
\frac{\plm(\yvar)^2}{\qlm(\yvar)^2}
\right]
-1
\right)
\\
&=
\estimand^2
\left(
\sum_{\str\in\kleene{\alphabet}}
\frac{\plm(\str)^2}{\qlm(\str)}
-1
\right).
\end{align*}
Finally,
\[
\sum_{\str\in\kleene{\alphabet}}
\frac{\plm(\str)^2}{\qlm(\str)}
\geq
\left(\sum_{\str\in\kleene{\alphabet}}\plm(\str)\right)^2
=1,
\]
with equality if and only if $\qlm=\plm$. Thus, when $\estimand\neq0$, the variance is zero if and only if $\qlm=\plm$.
\end{proof}
The control-variate estimator of \Cref{def:potential-cv} is formed with the baseline $\hest_*=\thIS$ and $\hest(\ppot)$ replaced by $\thIS(\ppot)$,
\begin{equation}
\label{eq:is-potential-cv}
\thIS(\ppot,\cvcoef)
\defeq
\thIS
-
\cvcoef\bigl(\thIS-\thIS(\ppot)\bigr),
\qquad
\cvcoef\in\R.
\end{equation}
The proof of \Cref{prop:optimal-alpha} uses only the unbiasedness and finite second moments of the two estimators, so it applies verbatim: $\thIS(\ppot,\cvcoef)$ is unbiased for every $\cvcoef$, and at $\astar$ its variance is at most that of $\thIS$.

\section{Additional Analysis and Experimental Details}
\label{app:experiments}

This appendix collects the setup behind the experiments of \Cref{sec:practice}: the models and prompt banks, the target-construction protocols, the generation configurations, and the supporting figures and tables. One subsection per application, in the order of \Cref{sec:practice}.

\subsection{Analysis of Watermark Detection}
\label{app:exp-watermark-analysis}
\label{app:exp-watermark}

A language model watermark embeds a statistical signal in generated text while preserving the model's intended task. Watermark detection then aims to reliably assess whether a text of a given length was generated by a watermarked model.  We use the scheme of \citet{kirchenbauer2023watermark}: at each generation step, a pseudorandom \defn{green list} $G(\xprefix)\subseteq\alphabet$ is determined from the prompt and generated prefix $\xprefix$, and the logits of tokens in this set are increased by a bias $\delta$. Let $\gamma$ denote the fraction of the vocabulary assigned to the green list, $m$ the generation length, and $\sigma\defeq\sqrt{m\gamma(1-\gamma)}$. Detection then counts the number of generated tokens that fall in their respective green lists and standardizes this count,

\begin{equation}
\label{eq:watermark-score}
\frac{|\yvar|_G-\gamma m}{\sigma},
\qquad
|\yvar|_G
\defeq
\#\{t:\strv_t\in G(\yvar_{<t})\},
\end{equation}
declaring the string watermarked when this exceeds a prescribed threshold
$\tau$, equivalently when $|\yvar|_G\ge\gamma m+\tau\sigma$.

\subsubsection{Estimands and potential definitions}
\begin{table}[!htb]
\refstepcounter{table}\label{tab:watermark-score}%
\refstepcounter{table}\label{tab:watermark-detection}%
\dtcard{@{}V@{\hspace{5pt}}l@{\hspace{7pt}}V@{\hspace{7pt}}l@{\hspace{5pt}}V@{}}{%
\toprule
\dtcardhead{Table~\ref{tab:watermark-score}: Expected watermark score}
& \dtcardhead{Table~\ref{tab:watermark-detection}: Probability of watermark detection} \\
\midrule
$\testf(\yvar)=\bigl(|\yvar|_G-\gamma m\bigr)/\sigma$
& $\testf(\yvar)=\1\{|\yvar|_G\ge\gamma m+\tau\sigma\}$\\
\midrule
$|\xprefix\zprefix|_G=|\xprefix|_G+|\zprefix|_G$
& not additive in the prefix\\
\midrule
$\ppot^\star(\xprefix)=\bigl(|\xprefix|_G+\E_{\plm}[|\zvar|_G\mid\zvar\succeq\xprefix]-\gamma m\bigr)/\sigma$
& $\ppot^\star(\xprefix)=\Pr_{\plm}\bigl[|\yvar|_G\ge\gamma m+\tau\sigma\mid\yvar\succeq\xprefix\bigr]$\\
\midrule
$\ppot(\xprefix)=\bigl(|\xprefix|_G-\gamma m\bigr)/\sigma$
& $\ppot(\xprefix,q)=\Pr\bigl[\mathrm{Bin}(m-|\xprefix|,q)\ge k(\xprefix)\bigr]$\\
\midrule
$\Delta\ppot(a\mid\xprefix)=\1\{a\in G(\xprefix)\}/\sigma$
& $\Delta\ppot(a\mid\xprefix,q)=\ppot(\xprefix a,q)-\ppot(\xprefix,q)$\\
\midrule
$\Delta\ppot(\eos\mid\xprefix)=0$
& $\Delta\ppot(\eos\mid\xprefix,q)=0$\\
\bottomrule}
\end{table}
At a fixed prompt and generation length, two expectations characterize the
detector. The \defn{expected score}
$\E_{\plm}\bigl[(|\yvar|_G-\gamma m)/\sigma\bigr]$ measures how strongly the
watermark registers on average; the \defn{detection probability}
$\Pr_{\plm}[|\yvar|_G\ge\gamma m+\tau\sigma]$ is the detector's true-positive
rate, and is the quantity a deployment actually cares about. The two behave
differently under our framework. The score has an additive structure over prefixes
(\Cref{tab:watermark-score}) so the value of the function at the observed prefix is
itself a potential, exactly as for the losses of \Cref{sec:app-losses}. The
detection indicator is not additive in the prefix, so it admits no such
decomposition and must instead be approached by approximating the oracle
potential, as in the event analysis of \Cref{sec:app-events}. Here that oracle
potential is the conditional probability that the detector will eventually fire.

\paragraph{Detection probability: analytical approximation of the oracle potential.}
For the detection probability a natural analytical approximation of the oracle
potential is available. Write
$k(\xprefix)=\gamma m+\tau\sigma-|\xprefix|_G$ for the number of green tokens
still needed at the prefix $\xprefix$. If the remaining $m-|\xprefix|$ tokens
were green independently with a common rate $q$, the detector would fire with
probability $\Pr[\mathrm{Bin}(m-|\xprefix|,q)\ge k(\xprefix)]$, which is the
potential of \Cref{tab:watermark-detection}. It depends on the single parameter
$q$, which we fit globally for each length $m$ by cross-fitting on $10$ folds of the continuations: the $q$ applied to one fold is the average green-list mass of the other nine, pooled over prompts. Computing this potential requires no additional model evaluations, as it uses only the green-list mass $\sum_{a\in G(\xprefix)} \pprob(a\mid\xprefix)$ from the already-computed next-token distribution.

\subsubsection{Experiments and results}
\begin{wraptable}[14]{r}{0.46\linewidth}
\vspace{-1.2\baselineskip}
\centering\scriptsize\setlength{\tabcolsep}{3.5pt}
\renewcommand{\arraystretch}{0.95}
\begin{tabular}{c c c c c}
\toprule
$\delta$ & $m$ & median & $[p_{10},\,p_{90}]$ & worst prompt \\
\midrule
$1$ & $32$ & $0.471$ & $[0.413,\,0.564]$ & $0.907$ \\
 & $64$ & $0.513$ & $[0.460,\,0.597]$ & $0.894$ \\
 & $128$ & $0.561$ & $[0.508,\,0.677]$ & $0.931$ \\
 & $200$ & $0.612$ & $[0.544,\,0.748]$ & $0.965$ \\
\midrule
$2$ & $32$ & $0.538$ & $[0.471,\,0.631]$ & $0.836$ \\
 & $64$ & $0.584$ & $[0.520,\,0.671]$ & $0.876$ \\
 & $128$ & $0.644$ & $[0.583,\,0.748]$ & $0.925$ \\
 & $200$ & $0.699$ & $[0.629,\,0.818]$ & $0.962$ \\
\midrule
$5$ & $32$ & $0.722$ & $[0.608,\,0.807]$ & $0.922$ \\
 & $64$ & $0.743$ & $[0.665,\,0.817]$ & $0.960$ \\
 & $128$ & $0.790$ & $[0.725,\,0.867]$ & $0.993$ \\
 & $200$ & $0.837$ & $[0.776,\,0.910]$ & $0.996$ \\
\bottomrule
\end{tabular}
\caption{Per-prompt variance of $\hest(\ppot)$ against Monte Carlo, $512$ prompts.}
\label{tab:wm-score-spread}
\end{wraptable}
\paragraph{Setup.}
We use the same setup of \citet{kirchenbauer2023watermark}, with green-list fraction $\gamma=0.5$ and logit biases $\delta\in\{1,2, 5\}$. For each configuration, we generate $4{,}000$ continuations per prompt from OPT-1.3B for $512$ prompts from C4 RealNewsLike, using temperature $0.7$. We evaluate continuation lengths $m\in\{32,64,128,200\}$ and detector thresholds $z_0\in\{3,4,5\}$. 

\begin{figure}[t]
\centering
\begin{minipage}[t]{0.48\textwidth}
\centering
{\footnotesize\textbf{(a)} binomial potential and green-count paths}\\[3pt]
\includegraphics[width=\linewidth]{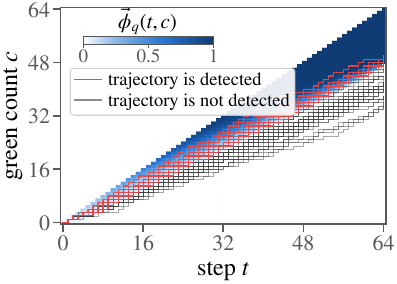}
\end{minipage}\hfill
\begin{minipage}[t]{0.48\textwidth}
\centering
{\footnotesize\textbf{(b)} variance ratio across regimes}\\[3pt]
\includegraphics[width=\linewidth]{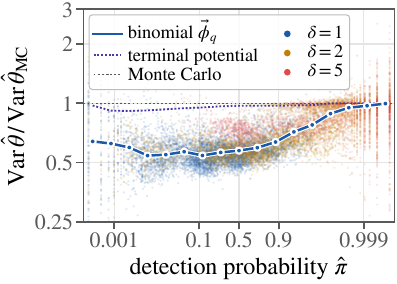}
\end{minipage}
\caption{\footnotesize Panel (a) shows how sampled trajectories update the green-token count over 64 generation steps. Red trajectories are those whose final green-token count exceeds the detection threshold, while black trajectories remain below it. The background encodes the potential’s assigned probability of eventual detection at each step and green-count state, with white indicating probability zero. Red trajectories tend to pass through higher-probability regions, whereas black trajectories concentrate in low- or zero-probability regions. Panel (b) shows the variance reduction at different detection probabilities}
\label{fig:watermark}
\end{figure}

For each of the two estimands, we compare the associated potentials against Monte Carlo. For the watermark-detection probability, we also include the baseline terminal potential, i.e., the corresponding observed-prefix potential for non-additive functions: zero at every nonterminal prefix and equal to the function value at the terminal prefix.

\paragraph{Detection probability: the binomial potential removes about $44\%$ of the Monte Carlo variance when detection is uncertain, but little once it saturates.}
\Cref{fig:watermark}(b) shows the variance ratio between the potential estimator and MC for each prompt and configuration $(\delta,m,z_0)$. Across the $14{,}366$ pairs for which detection is not deterministic, the median ratio is $0.64$, corresponding to a $36\%$ reduction in variance. Since variance scales inversely with the number of samples, this is equivalent to reaching the precision of $4{,}000$ Monte Carlo samples with about $2{,}560$ samples, without additional model evaluations. The gain is largest when detection is uncertain: for $0.001\le\hat\pi<0.5$, the median ratio is $0.56$, corresponding to a $44\%$ reduction. It decreases as detection saturates, reaching a median ratio of $0.97$ for $\hat\pi\ge0.99$. Overall, $92.7\%$ of pairs improve over Monte Carlo, with similar behavior across $\delta$, $m$, and $z_0$. The terminal potential, by contrast, remains close to Monte Carlo, with a median ratio of $0.973$.

\Cref{fig:watermark}(a) gives intuition for the variance reduction. From early in generation, trajectories that are eventually detected overlap with regions of positive potential, whereas undetected trajectories lie mostly in the white region where the potential is zero. Thus, the potential already contains information about the final detection outcome long before the trajectory ends.

\paragraph{Expected score: every prompt improves, by $16$--$53\%$ depending on $\delta$ and $m$.}
\Cref{tab:wm-score-spread} gives the per-prompt variance relative to Monte Carlo over the $512$ prompts. The
median ratio rises from $0.471$ at $\delta=1,m=32$ to $0.837$ at $\delta=5,m=200$: the marginalization helps most when
the watermark is weak and the continuation is short, the regime in which a single green indicator carries the most
conditional variance. The spread over prompts is narrow, the $p_{10}$--$p_{90}$ width never exceeding $0.21$, and every
one of the $512$ prompts lies below Monte Carlo in all twelve cells, the worst single prompt reaching $0.996$. We also cross-fit a control-variate
coefficient on top of but observe that it buys almost nothing here: across the twelve cells the  variance against
$\hest(\ppot)$ stays within $0.991$--$0.998$ and the median prompt within $0.994$--$1.000$.

\subsection{Losses and Divergences: Details and Further Results}
\label{app:exp-details}

\subsubsection{Potential definitions}
\label{app:proper-quantities}
Below we give the potential tables for the losses and divergences of \Cref{sec:app-losses}; the corresponding gradients are in \Cref{app:exp-gradients}. {Continuation sums include termination, with $\yvar_{|\yvar|+1}=\zprefix_{|\zprefix|+1}=\eos$.}
\medskip
\refstepcounter{table}\label{tab:kl}%
\dtcard{@{}V@{\hspace{5pt}}l@{\hspace{5pt}}V l@{}}{%
\toprule
\dtcardhead{Table~\thetable: KL divergence} & \\
\midrule
$\testf(\yvar)=\log\dfrac{\plm(\yvar)}{\qlm(\yvar)}$
& \rowtag{Test function}\\
\midrule

{$\ppot^\star(\xprefix)=\log\frac{\pprob(\xprefix)}{\qpprob(\xprefix)}+\E_{\plm}\!\left[\log\frac{\plm(\yvar)\,\qpprob(\xprefix)}{\qlm(\yvar)\,\pprob(\xprefix)}\;\middle|\;\yvar\succeq\xprefix\right]$}
& \rowtag{oracle potential}\\
\midrule
$\ppot(\xprefix)=\log\dfrac{\pprob(\xprefix)}{\qpprob(\xprefix)}$
& \rowtag{Potential}\\
\midrule
$\Delta\ppot(a\mid\xprefix)=\log\dfrac{\pprob(a\mid\xprefix)}{\qpprob(a\mid\xprefix)}$
& \rowtag{Increment ($a\in\alphabet$)}\\
\midrule
$\Delta\ppot(\eos\mid\xprefix)=\log\dfrac{\pprob(\eos\mid\xprefix)}{\qpprob(\eos\mid\xprefix)}$
& \rowtag{\eos{} increment}\\
\midrule
$\mathbf B\Delta\ppot(\xprefix)=\KL\bigl(\pprob(\cdot\mid\xprefix)\,\|\,\qpprob(\cdot\mid\xprefix)\bigr)$
& \rowtag{Expected increment}\\
\bottomrule}

\medskip
\refstepcounter{table}\label{tab:brier}%
\dtcard{@{}V@{\hspace{5pt}}l@{\hspace{5pt}}V@{}}{%
\toprule
\dtcardhead{Table~\thetable: Brier loss} \\
\midrule
$\testf(\yvar)=\sum_{t}\sum_{b\in\alphabar}\bigl(\pprob(b\mid\yvar_{<t})-\1\{b=\strv_t\}\bigr)^2$\\
\midrule
$\Delta\ppot(a\mid\xprefix)=\sum_{b\in\alphabar}\bigl(\pprob(b\mid\xprefix)-\1\{b=a\}\bigr)^2=1-2\,\pprob(a\mid\xprefix)+\sum_{b\in\alphabar}\pprob(b\mid\xprefix)^2$\\
\midrule
{$\testf(\xprefix\zprefix)=\ppot(\xprefix)+\sum_{t=1}^{|\zprefix|+1}\Delta\ppot(\zprefix_t\mid\xprefix\zprefix_{<t})$}\\
\midrule
{$\ppot^\star(\xprefix)=\ppot(\xprefix)+\E_{\plm}\!\left[\sum_{t=|\xprefix|+1}^{|\yvar|+1}\Delta\ppot(\yvar_t\mid\yvar_{<t})\;\middle|\;\yvar\succeq\xprefix\right]$}\\
\midrule
$\ppot(\xprefix)=\sum_{t\le|\xprefix|}\Delta\ppot(\xprefix_t\mid\xprefix_{<t})$\\
\midrule
$\Delta\ppot(\eos\mid\xprefix)=1-2\,\pprob(\eos\mid\xprefix)+\sum_{b\in\alphabar}\pprob(b\mid\xprefix)^2$\\
\midrule
$\mathbf B\Delta\ppot(\xprefix)=1-\sum_{b\in\alphabar}\pprob(b\mid\xprefix)^2$\\
\bottomrule}

\medskip
\refstepcounter{table}\label{tab:power}%
\dtcard{@{}V@{\hspace{5pt}}l@{\hspace{5pt}}V@{}}{%
\toprule
\dtcardhead{Table~\thetable: Power-score regret, exponent $\beta$} \\
\midrule
$\testf(\yvar)=\sum_{t}\Bigl[\beta\bigl(\pprob(\strv_t\mid\yvar_{<t})^{\beta-1}-\qpprob(\strv_t\mid\yvar_{<t})^{\beta-1}\bigr)-(\beta-1)\sum_{b\in\alphabar}\bigl(\pprob(b\mid\yvar_{<t})^{\beta}-\qpprob(b\mid\yvar_{<t})^{\beta}\bigr)\Bigr]$\\
\midrule
$\Delta\ppot(a\mid\xprefix)=\beta\bigl(\pprob(a\mid\xprefix)^{\beta-1}-\qpprob(a\mid\xprefix)^{\beta-1}\bigr)-(\beta-1)\sum_{b\in\alphabar}\bigl(\pprob(b\mid\xprefix)^{\beta}-\qpprob(b\mid\xprefix)^{\beta}\bigr)$\\
\midrule
{$\testf(\xprefix\zprefix)=\ppot(\xprefix)+\sum_{t=1}^{|\zprefix|+1}\Delta\ppot(\zprefix_t\mid\xprefix\zprefix_{<t})$}\\
\midrule
{$\ppot^\star(\xprefix)=\ppot(\xprefix)+\E_{\plm}\!\left[\sum_{t=|\xprefix|+1}^{|\yvar|+1}\Delta\ppot(\yvar_t\mid\yvar_{<t})\;\middle|\;\yvar\succeq\xprefix\right]$}\\
\midrule
$\ppot(\xprefix)=\sum_{t\le|\xprefix|}\Delta\ppot(\xprefix_t\mid\xprefix_{<t})$\\
\midrule
$\Delta\ppot(\eos\mid\xprefix)=\beta\bigl(\pprob(\eos\mid\xprefix)^{\beta-1}-\qpprob(\eos\mid\xprefix)^{\beta-1}\bigr)-(\beta-1)\sum_{b\in\alphabar}\bigl(\pprob(b\mid\xprefix)^{\beta}-\qpprob(b\mid\xprefix)^{\beta}\bigr)$\\
\midrule
$\mathbf B\Delta\ppot(\xprefix)=\sum_{b\in\alphabar}\Bigl[\pprob(b\mid\xprefix)^{\beta}-\beta\,\pprob(b\mid\xprefix)\,\qpprob(b\mid\xprefix)^{\beta-1}+(\beta-1)\,\qpprob(b\mid\xprefix)^{\beta}\Bigr]$\\
\bottomrule}

\medskip
\noindent Entropy is the same construction with $\Delta\ppot(a\mid\xprefix)=-\log\pprob(a\mid\xprefix)$, shown in \Cref{tab:entropy}. The four power-score regrets use $\beta\in\{1.25,1.5,1.75,2\}$; as $\beta\to1$ the per-token term, divided by $\beta-1$, tends to the KL term of \Cref{tab:kl}. 

\subsubsection{Experimental setup}
\label{app:proper-pairs}

We evaluate the losses and divergences on eight configurations, each
defined by a sampling policy, a reference model, and a generation horizon
(\Cref{tab:proper-pairs}). These cover seven distinct policy--reference
pairs: the GPT-2 sentiment pair is evaluated at two horizons, one of $50$
tokens that reuses the setting of \citet{amini2025better} for reproducibility,
and one of $256$ tokens to match the longer window used for all other models.
Model repositories are listed in \Cref{tab:proper-repos}.

\paragraph{Shared sampling protocol.}
Each configuration uses $512$ prompts and $B$ responses per prompt,
with $B$ specified in \Cref{tab:proper-pairs}. Responses are drawn by
ancestral sampling at temperature $1$, without top-$k$ or top-$p$
filtering, using fp32 precision. Generation stops at \eos{} or the
specified horizon.
\begin{table}[H]
\centering
\caption{Policy--reference banks of the proper-loss experiment. Labels
are those used in the figures; $\plm$ is the sampling policy and $\qlm$ the
reference; model names are resolved in \Cref{tab:proper-repos}. All banks have
$512$ prompts and $B$ responses per prompt.}
\label{tab:proper-pairs}
\footnotesize
\setlength{\tabcolsep}{3pt}
\begin{tabular}{@{}llllrr@{}}
\toprule
bank & prompts & sampling policy $\plm$ & reference $\qlm$ & horizon & $B$ \\
\midrule
IMDb-256 & IMDb prefixes & gpt2-imdb + DPO & gpt2-imdb & 256 & 4000 \\
IMDb-50 & IMDb prefixes & same policy & gpt2-imdb & 50 & 4000 \\
T\"ulu-3 & AlpacaEval 2 & Tulu-3-8B-DPO & Tulu-3-8B-SFT & 256 & 100 \\
Gemma It3 vs It0 & AlpacaEval 2 & SPPO-Iter3 & gemma-2-9b-it & 256 & 100 \\
Gemma It1 vs It0 & AlpacaEval 2 & SPPO-Iter1 & gemma-2-9b-it & 256 & 100 \\
Gemma It2 vs It0 & AlpacaEval 2 & SPPO-Iter2 & gemma-2-9b-it & 256 & 100 \\
Gemma It2 vs It1 & AlpacaEval 2 & SPPO-Iter2 & SPPO-Iter1 & 256 & 100 \\
Gemma It3 vs It2 & AlpacaEval 2 & SPPO-Iter3 & SPPO-Iter2 & 256 & 100 \\
\bottomrule
\end{tabular}
\end{table}

\paragraph{GPT-2-imbd + DPO: policy training.}
We train the policy $p$ ourself, following the  setting of \citet{amini2025better}.
The reference $\qlm$ is \texttt{gpt2-imdb}; we obtain the policy
$\plm$ by fine-tuning it for positive sentiment with direct preference
optimization (DPO) on $5{,}000$ IMDb training reviews. Training prompts
are $2$--$8$-token prefixes. Four sampled responses per prompt form
six response pairs, with preferences determined by
\texttt{distilbert-imdb}. We train for one epoch ($938$ steps), using
$\beta_{\mathrm{DPO}}=0.1$ and a learning rate of $3\times10^{-6}$.

\paragraph{GPT-2 sentiment: evaluation prompts.}
We select $512$ reviews from the IMDb test split using a fixed seed.
For each review, we draw a prefix length uniformly from
$\{2,\ldots,8\}$ and retain the corresponding GPT-2 token IDs as the
prompt. We draw $4{,}000$ responses per prompt at a primary horizon
of $256$ tokens and a control horizon of $50$ tokens, using the same
policy and reference in both configurations.

\paragraph{Chat models: policy--reference comparisons.}
For T\"ulu-3, we compare the published DPO policy with its SFT
reference. For Gemma-2, we use the three published SPPO iterations,
denoting the initial instruction model by iteration $0$.
Comparing iterations $1$, $2$, and $3$ with iteration $0$ measures
accumulated drift. We additionally compare iteration $2$ with $1$
and iteration $3$ with $2$ to examine local drift between successive
updates.

\paragraph{Chat models: evaluation prompts.}
The chat-model configurations use $512$ instructions from the
$805$-instruction AlpacaEval~2 set.  Instructions are rendered with the corresponding model's
chat template; within each policy--reference pair, both models
receive identical token IDs. Each prompt receives $100$ responses,
with a horizon of $256$ tokens.

\begin{table}[H]
\centering
\caption{Model repositories (Hugging Face) for the names used in
\Cref{tab:proper-pairs}.}
\label{tab:proper-repos}
\footnotesize
\begin{tabular}{@{}ll@{}}
\toprule
model & repository \\
\midrule
gpt2-imdb & \texttt{lvwerra/gpt2-imdb} \\
gpt2-imdb + DPO & trained by us from \texttt{lvwerra/gpt2-imdb} (recipe above) \\
distilbert-imdb (DPO preference judge) & \texttt{lvwerra/distilbert-imdb} \\
Tulu-3-8B-SFT & \texttt{allenai/Llama-3.1-Tulu-3-8B-SFT} \\
Tulu-3-8B-DPO & \texttt{allenai/Llama-3.1-Tulu-3-8B-DPO} \\
gemma-2-9b-it (SPPO Iter0) & \texttt{google/gemma-2-9b-it} \\
SPPO-Iter1 & \texttt{UCLA-AGI/Gemma-2-9B-It-SPPO-Iter1} \\
SPPO-Iter2 & \texttt{UCLA-AGI/Gemma-2-9B-It-SPPO-Iter2} \\
SPPO-Iter3 & \texttt{UCLA-AGI/Gemma-2-9B-It-SPPO-Iter3} \\
\bottomrule
\end{tabular}
\end{table}

\subsubsection{Additional results}
\label{app:proper-results}

\paragraph{Variance reduction across quantities.}
\Cref{fig:exp1-magnitude-all} extends the magnitude--variance
comparison of \Cref{fig:exp1-proper}(b) to all seven quantities.
For each quantity, it relates the variance ratio to the estimated
magnitude of the quantity, showing both configuration-level
aggregates and individual prompts.

\begin{figure}[H]
\centering
\includegraphics[width=\textwidth]{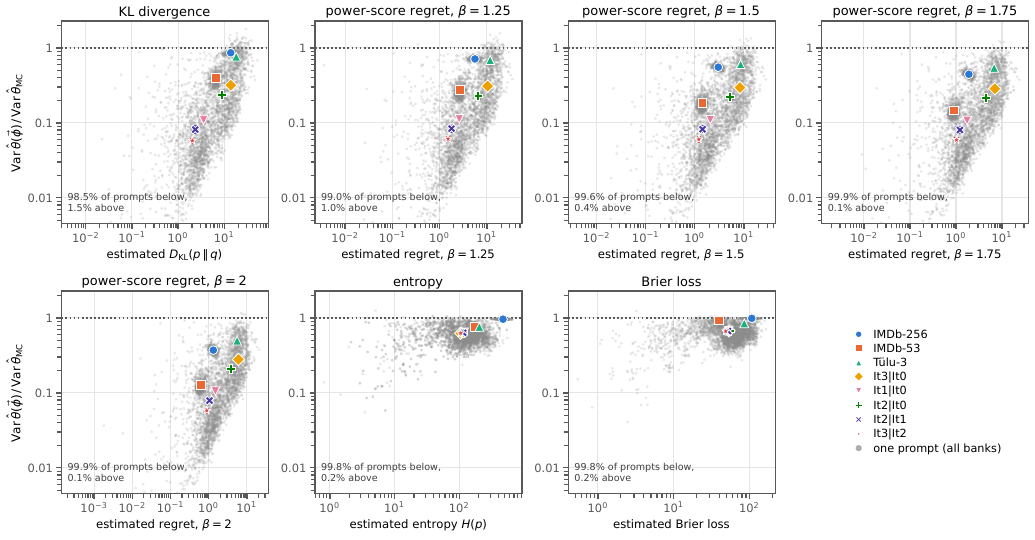}
\caption{\footnotesize Variance reduction as a function of quantity magnitude for all losses and divergences in the proper-loss experiment. Each panel shows the variance ratio $\Var_{\plm}(\hest(\ppot))/\Var_{\plm}(\thmc)$ against the estimated magnitude of the quantity, with one colored marker per policy--reference bank and grey points for individual prompts. For KL and the power-score regrets, variance reduction becomes stronger as the policy--reference divergence decreases, with this dependence becoming more pronounced for larger $\beta$. Entropy and Brier loss instead remain comparatively close to the Monte Carlo baseline across banks.}

\label{fig:exp1-magnitude-all}
\end{figure}

\paragraph{Prompt-level behavior and control variates.}
\Cref{tab:prompt-cv} summarizes the per-prompt variance ratios of
the potential estimator and its cross-fitted control-variate
variant relative to Monte Carlo. It also reports separately the
prompts on which the potential alone increases variance, showing
how fitting the control-variate coefficient changes their ratios
and how often it brings them below one.

\begin{table}[H]
\centering\footnotesize
\setlength{\tabcolsep}{3pt}
\resizebox{\linewidth}{!}{%
\begin{tabular}{@{}l cc cc r c c c@{}}
\toprule
& \multicolumn{2}{c}{$\hest(\ppot)$ vs.\ $\thmc$} & \multicolumn{2}{c}{$\hest(\ppot,\astar)$ vs.\ $\thmc$} & \multicolumn{4}{c}{prompts with $\Var_{\plm}(\hest(\ppot))>\Var_{\plm}(\thmc)$} \\
\cmidrule(lr){2-3}\cmidrule(lr){4-5}\cmidrule(l){6-9}
quantity & median [IQR] & \% $<1$ & median [IQR] & \% $<1$ & $n$ & \begin{tabular}[c]{@{}c@{}}med.\\[-1pt]$\hest(\ppot)$\end{tabular} & \begin{tabular}[c]{@{}c@{}}med.\ $\hest(\ppot,\astar)$\\[-1pt][95\% CI]\end{tabular} & \% $<1$ \\
\midrule
$\KL$ & 0.204 [0.061, 0.503] & 98.5 & 0.204 [0.061, 0.482] & 99.5 & 63 & 1.134 & 0.974 [0.945, 0.990] & 65.1 \\
regret $\beta{=}1.25$ & 0.202 [0.063, 0.441] & 99.0 & 0.193 [0.063, 0.436] & 99.8 & 40 & 1.083 & 0.951 [0.932, 0.973] & 77.5 \\
regret $\beta{=}1.5$ & 0.164 [0.063, 0.400] & 99.6 & 0.160 [0.063, 0.396] & 99.9 & 15 & 1.058 & 0.935 [0.891, 1.044] & 66.7 \\
regret $\beta{=}1.75$ & 0.142 [0.062, 0.365] & 99.9 & 0.141 [0.061, 0.363] & 99.9 & 4 & 1.190 & 1.080 [0.925, 1.102] & 25.0 \\
regret $\beta{=}2$ & 0.131 [0.060, 0.330] & 99.9 & 0.129 [0.060, 0.329] & 99.9 & 3 & 1.125 & 1.040 [1.009, 1.043] & 0.0 \\
entropy & 0.633 [0.509, 0.796] & 99.8 & 0.630 [0.498, 0.792] & 99.5 & 10 & 1.047 & 1.090 [1.050, 1.096] & 0.0 \\
Brier loss & 0.694 [0.545, 0.903] & 99.8 & 0.698 [0.543, 0.903] & 99.5 & 9 & 1.040 & 1.059 [1.024, 1.080] & 0.0 \\
\bottomrule
\end{tabular}}
\caption{Per-prompt variance ratios for the losses and divergences experiments, pooling all eight model pairs. We compare the potential estimator $\hest(\ppot)$ and its cross-fitted control-variate variant $\hest(\ppot,\astar)$ against Monte Carlo. The rightmost columns restrict attention to prompts for which $\hest(\ppot)$ has higher variance than Monte Carlo and show how often the control variate reduces the ratio below one.}

\label{tab:prompt-cv}
\end{table}

\subsubsection{Reproducing \citet{amini2025better}}
\label{app:exp-kl}
We first reproduce the KL-estimation experiment of \citet{amini2025better} using our
sentiment-tuned policy and compare against the estimators considered in their
paper. Their term-wise Rao--Blackwellized estimator is recovered as
$\hest({\ppot})$, corresponding to a unit coefficient on the associated
zero-mean correction. We reproduce its variance reduction over
Monte Carlo (\Cref{tab:kl-estimators}). Fitting the control-variate coefficient reduces variance further:
at $M=1$, the standard deviation decreases from $2.05$ to $1.90$. Thus, the
unit coefficient is not in general variance-optimal.
\begin{table}[H]
\centering\small
\setlength{\tabcolsep}{5pt}
\begin{tabular}{@{}lrrrr@{}}
\toprule
& & \multicolumn{3}{c}{std.\ dev.} \\
\cmidrule(lr){3-5}
estimator & mean & $M{=}1$ & $M{=}5$ & $M{=}10$ \\
\midrule
$\thmc$                         & 6.69 & 3.32 & 1.49 & 1.05 \\
HT                              & 6.69 & 3.32 & 1.49 & 1.05 \\
Schulman CV, $\cvcoef=1$        & 6.47 & 35.67 & 15.95 & 11.28 \\
Schulman CV, fitted $\cvcoef$   & 6.36 & 11.44 & 5.11 & 3.61 \\
$\hest(\ppot)$                  & 6.69 & 2.05 & 0.92 & 0.65 \\
$\hest(\ppot,\astar)$           & 6.69 & $\mathbf{1.90}$ & $\mathbf{0.85}$ & $\mathbf{0.60}$ \\
\bottomrule
\end{tabular}
\caption{Reproduction of the KL-estimation experiment of \citet{amini2025better}:
mean estimate and empirical standard deviation (nats) at sample sizes
$M=1,5,10$, over $512$ prompts with $4{,}000$ samples each. The mean does not
depend on $M$. Lowest standard deviation per column in bold.}
\label{tab:kl-estimators}
\end{table}

\paragraph{Counterexample to the variance guarantee of \citet{amini2025better}.}
\label{app:counterexample}
Theorem~2 of \citet{amini2025better} claims that their term-wise Rao--Blackwellized estimator never has larger variance than Monte Carlo. Consider language models $\plm$ and $\qlm$ supported on length-two strings, with

$$
\plm(ax)=\plm(bx)=\tfrac12,
\qquad
\qlm(ax)=\qlm(bx)=\tfrac14,
\qquad
\qlm(bz)=\tfrac12,
$$

and zero probability elsewhere. Under $\plm$, $\log \frac{\plm(Y)}{\qlm(Y)}=\log2$ for both possible strings, so $\Var_{\plm}(\thmc)=0$. Their estimator, corresponding to $\hest(\ppot)$ in our notation, instead equals $\frac12\log\frac43$ or $\frac12\log\frac43+\log3$ with equal probability. It remains unbiased, but

$$
\Var_{\plm}\!\bigl(\hest(\ppot)\bigr)=\tfrac14(\log3)^2>0=\Var_{\plm}(\thmc).
$$

This also follows directly from \Cref{prop:potential-estimator}: since $\thmc$ is constant,
$\Cov_{\plm}(\thmc,\thmc-\hest(\ppot))=0<\frac12\Var_{\plm}(\thmc-\hest(\ppot))$, so the condition for variance reduction fails. Accordingly, \Cref{prop:optimal-alpha} gives $\astar=0$, whereas the estimator of \citet{amini2025better} corresponds to $\cvcoef=1$.

\subsection{Gradients of Losses and Divergences}
\label{app:exp-gradients}

\subsubsection{Potential definitions}
\label{app:grad-potentials}

Let $\mathbf w\in\R^d$ denote the parameters of $\plm$, and write
$\nabla\defeq\nabla_{\mathbf w}$, holding $\qlm$ fixed. The original test
function $\testf$ is scalar-valued, whereas its gradient with respect to
$\mathbf w$ is $d$-dimensional. 
This gives the vector-valued test function
\begin{equation}
\label{eq:grad-test-function}
\boldsymbol{\testf}_\nabla(\yvar)
\defeq
\nabla\testf(\yvar)
+
\testf(\yvar)\,\nabla\log\plm(\yvar)
\in\R^d.
\end{equation}
The first term differentiates the function value for a fixed string, while
the second accounts for the change in its sampling probability. Under
standard regularity conditions,
\begin{equation}
\label{eq:grad-score-function}
\nabla\,\E_{\plm}[\testf(\yvar)]
=
\E_{\plm}[\boldsymbol{\testf}_\nabla(\yvar)].
\end{equation}
Gradient estimation therefore fits our framework by applying the potential
estimator to $\boldsymbol{\testf}_\nabla$ coordinatewise. Given a scalar potential $\ppot$ for $\testf$, we define the corresponding
vector-valued gradient potential
\begin{equation}
\label{eq:grad-potential}
\boldsymbol{\ppot}_\nabla(\xprefix)
\defeq
\nabla\ppot(\xprefix)
+
\ppot(\xprefix)\,\nabla\log\pprob(\xprefix)
\in\R^d.
\end{equation}
Replacing $\ppot$ by $\ppot^\star$ gives the corresponding oracle gradient
potential.

Because the next-token probabilities in $\mathbf B$ also depend on
$\mathbf w$, the expected gradient increment satisfies
\begin{equation}
\label{eq:grad-expected-increment}
\mathbf B\Delta\boldsymbol{\ppot}_\nabla(\xprefix)
=
\nabla\mathbf B\Delta\ppot(\xprefix)
+
\mathbf B\Delta\ppot(\xprefix)\,
\nabla\log\pprob(\xprefix).
\end{equation}
We extend the operators $\Delta$ and $\mathbf B$ to vector-valued potentials
coordinatewise. \Cref{tab:grad} gives the resulting constructions for the losses in
\Cref{tab:entropy,tab:kl,tab:brier,tab:power}. For KL,
$\mathbf B\Delta\ppot(\xprefix)$ is the KL divergence between the
next-token distributions, recovering the local gradient expression of
\citet{amini2025better}.

\medskip
\begingroup
\setlength{\dtcardextra}{0pt}\setlength{\dttaggutter}{78pt}
\refstepcounter{table}\label{tab:grad}%
\dtcard{@{}V@{\hspace{5pt}}l@{\hspace{5pt}}V l@{}}{%
\toprule
\dtcardhead{Table~\thetable: Gradient of a loss, $\nabla\,\E_{\plm}[\testf(\yvar)]$} & \\
\midrule
$\boldsymbol{\testf}_\nabla(\yvar)=\nabla\testf(\yvar)+\testf(\yvar)\,\nabla\log\plm(\yvar)$
& \rowtag{Test function}\\
\midrule
$\begin{aligned}[t]
\Delta\boldsymbol{\ppot}_\nabla(a\mid\xprefix)&=\nabla\Delta\ppot(a\mid\xprefix)+\Delta\ppot(a\mid\xprefix)\,\nabla\log\pprob(\xprefix a)\\
&\quad+\ppot(\xprefix)\,\nabla\log\pprob(a\mid\xprefix)
\end{aligned}$
& \rowtag{Increment ($a\in\alphabet$)}\\
\midrule
$\begin{aligned}[t]
\boldsymbol{\testf}_\nabla(\xprefix\zprefix)&=\nabla\ppot(\xprefix)+\ppot(\xprefix)\,\nabla\log\pprob(\xprefix)\\
&\quad+\nabla\bigl[\testf(\xprefix\zprefix)-\ppot(\xprefix)\bigr]+\bigl[\testf(\xprefix\zprefix)-\ppot(\xprefix)\bigr]\,\nabla\log\plm(\xprefix\zprefix)\\
&\quad+\ppot(\xprefix)\,\nabla\log\pprob(\zprefix\mid\xprefix)
\end{aligned}$
& \rowtag{Function structure}\\
\midrule
$\boldsymbol{\ppot}^\star_\nabla(\xprefix)=\nabla\ppot^\star(\xprefix)+\ppot^\star(\xprefix)\,\nabla\log\pprob(\xprefix)$
& \rowtag{oracle potential}\\
\midrule
$\boldsymbol{\ppot}_\nabla(\xprefix)=\nabla\ppot(\xprefix)+\ppot(\xprefix)\,\nabla\log\pprob(\xprefix)$
& \rowtag{Potential}\\
\midrule
$\Delta\boldsymbol{\ppot}_\nabla(\eos\mid\xprefix)=\nabla\Delta\ppot(\eos\mid\xprefix)+\Delta\ppot(\eos\mid\xprefix)\,\nabla\log\plm(\xprefix)+\ppot(\xprefix)\,\nabla\log\pprob(\eos\mid\xprefix)$
& \rowtag{\eos{} increment}\\
\midrule
$\mathbf B\Delta\boldsymbol{\ppot}_\nabla(\xprefix)=\nabla\mathbf B\Delta\ppot(\xprefix)+\mathbf B\Delta\ppot(\xprefix)\,\nabla\log\pprob(\xprefix)$
& \rowtag{Expected increment}\\
\bottomrule}
\endgroup

\subsubsection{Control-variate coefficient for vector-valued functions}
\label{app:grad-cv}

For the gradient we need to fit a global control-variate coefficient when the
function value lies in $\R^{d}$. We describe the construction for a general
vector-valued test function $\boldsymbol{\testf}\colon\kleene{\alphabet}\to\R^{d}$ with a
potential $\boldsymbol{\ppot}\colon\kleene{\alphabet}\to\R^{d}$, both read
coordinatewise. Let $\boldsymbol{\hest}_{\mathrm{MC}},\,\boldsymbol{\hest}(\boldsymbol{\ppot})\in\R^{d}$ denote the Monte
Carlo and potential estimators for one trajectory, and define the zero-mean
correction
\begin{equation}
\boldsymbol{D}
\defeq
\boldsymbol{\hest}_{\mathrm{MC}}-\boldsymbol{\hest}(\boldsymbol{\ppot}).
\end{equation}
As in \Cref{prop:optimal-alpha}, we use a single scalar coefficient, here shared
across all $d$ coordinates:
\begin{equation}
\boldsymbol{\hest}(\boldsymbol{\ppot},\cvcoef)
\defeq
\boldsymbol{\hest}_{\mathrm{MC}}-\cvcoef\,\boldsymbol{D}.
\end{equation}
Thus, $\cvcoef=0$ recovers Monte Carlo and $\cvcoef=1$ the potential estimator.
We choose $\cvcoef$ to minimize the total variance
$\tr\Var_{\plm}\bigl(\boldsymbol{\hest}(\boldsymbol{\ppot},\cvcoef)\bigr)$, giving
\begin{equation}
\astar
=
\frac{\tr\Cov_{\plm}\bigl(\boldsymbol{\hest}_{\mathrm{MC}},\boldsymbol{D}\bigr)}{\tr\Var_{\plm}(\boldsymbol{D})}
=
\frac{
\E_{\plm}\!\left[
\left\langle
\boldsymbol{\hest}_{\mathrm{MC}}-\E_{\plm}[\boldsymbol{\hest}_{\mathrm{MC}}],\;
\boldsymbol{D}-\E_{\plm}[\boldsymbol{D}]
\right\rangle
\right]
}{
\E_{\plm}\!\left[
\|\boldsymbol{D}-\E_{\plm}[\boldsymbol{D}]\|_2^2
\right]
}.
\end{equation}
Given sampled trajectories indexed by prompt $i$ and response $j$, we estimate
this globally as
\begin{equation}
\widehat{\cvcoef}
=
\frac{
\sum_{i,j}
\left\langle
\boldsymbol{\hest}_{\mathrm{MC}}^{(ij)}-\bar{\boldsymbol{\hest}}_{\mathrm{MC}},\;
\boldsymbol{D}^{(ij)}-\bar{\boldsymbol{D}}
\right\rangle
}{
\sum_{i,j}
\|\boldsymbol{D}^{(ij)}-\bar{\boldsymbol{D}}\|_2^2
},
\end{equation}
where the bars denote means over all trajectories. The coefficient is therefore
shared across prompts and coordinates. For the gradient experiments we take
$\boldsymbol{\testf}=\boldsymbol{\testf}_\nabla$ and $\boldsymbol{\ppot}=\boldsymbol{\ppot}_\nabla$ from
\Cref{app:grad-potentials}, with $d$ the number of policy parameters.

\subsubsection{Experiments and results}
\label{app:grad-results}

\textbf{Marginalization reduces gradient variance across all 135 defined quantity–checkpoint comparisons.} We repeat the comparison for the gradient of each quantity with respect to the policy parameters along a DPO training trajectory of the GPT-2 sentiment model (\Cref{fig:exp2-grad}). We evaluate $20$ checkpoints using $512$ prompts and $20$ responses per prompt, and measure the variance of a gradient estimator $\boldsymbol{\hest}$ by $\tr\Var_{\plm}(\boldsymbol{\hest})$. Across all $135$ quantity--checkpoint comparisons, the potential estimator $\boldsymbol{\hest}(\boldsymbol{\ppot}_\nabla)$ reduces variance relative to the centered Monte Carlo gradient $\boldsymbol{\hest}_{\mathrm{MC}}$, with variance ratios $\tr\Var_{\plm}(\boldsymbol{\hest}(\boldsymbol{\ppot}_\nabla))/\tr\Var_{\plm}(\boldsymbol{\hest}_{\mathrm{MC}})$ between $0.11$ and $0.98$.

\textbf{The gains are largest early in training.} The gains are largest early in training and generally decrease as training progresses. For example, the KL gradient retains $35\%$ of the Monte Carlo variance at step $50$ and $92\%$ at step $938$, while the $\beta=2$ regret retains $11\%$ and $65\%$, respectively. Among the power-score regrets, larger $\beta$ consistently yields larger reductions at every checkpoint.

\textbf{Fitting $\astar$ further improves gradients for entropy and Brier score.} For some gradients, the control-variate estimator $\boldsymbol{\hest}(\boldsymbol{\ppot}_\nabla,\astar)$ provides an additional gain over $\boldsymbol{\hest}(\boldsymbol{\ppot}_\nabla)$. The unit-coefficient estimator $\boldsymbol{\hest}(\boldsymbol{\ppot}_\nabla)$ alone barely improves the entropy and Brier gradients, with variance ratios of $0.97$--$0.98$. Their fitted coefficients are substantially larger, $\astar\approx5$ for entropy and $\astar\approx11$ for Brier, reducing the ratios to $0.91$ and $0.86$. For the power-score regrets, by contrast, $\astar$ remains close to one ($0.66$--$0.99$), and fitting it provides at most about a further $3\%$ reduction.

\begin{figure}[H]
\centering
\begin{minipage}[t]{0.33\textwidth}
\centering
{\footnotesize\textbf{(a)} gradient variance ratio}\\[3pt]
\includegraphics[width=\linewidth]{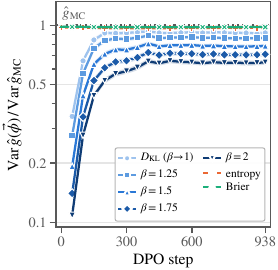}
\end{minipage}\hfill
\begin{minipage}[t]{0.33\textwidth}
\centering
{\footnotesize\textbf{(b)} cross-fitted control variate}\\[3pt]
\includegraphics[width=\linewidth]{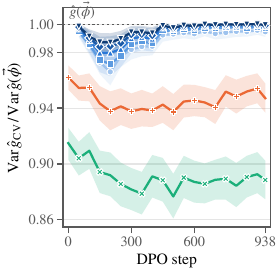}
\end{minipage}\hfill
\begin{minipage}[t]{0.33\textwidth}
\centering
{\footnotesize\textbf{(c)} fitted coefficient $\astar$}\\[3pt]
\includegraphics[width=\linewidth]{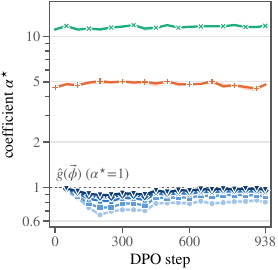}
\end{minipage}
\caption{\footnotesize Gradient estimation along a DPO trajectory of the GPT-2 sentiment
policy (checkpoints every $50$ steps; the variance of an estimator $\boldsymbol{\hest}$ is
$\tr\Var_{\plm}(\boldsymbol{\hest})$ over the $512\times20$ responses). \textbf{(a)} Variance of
$\boldsymbol{\hest}(\boldsymbol{\ppot}_\nabla)$ relative to the centered Monte Carlo gradient $\boldsymbol{\hest}_{\mathrm{MC}}$, one series per
quantity (regrets darker for larger $\beta$; undefined at step $0$).
\textbf{(b)} Variance of the cross-fitted control-variate estimator $\boldsymbol{\hest}(\boldsymbol{\ppot}_\nabla,\astar)$
relative to $\boldsymbol{\hest}(\boldsymbol{\ppot}_\nabla)$. \textbf{(c)} The cross-fitted coefficient, mean over
$8$ prompt folds with the range as band; $\astar=1$ recovers $\boldsymbol{\hest}(\boldsymbol{\ppot}_\nabla)$. Bands in (a) and (b) are $95\%$ prompt-cluster bootstrap intervals;
one training run and one seed.}
\label{fig:exp2-grad}
\end{figure}

\subsection{$n$-gram Counts and Frequencies}
\label{app:exp-ngram}

\subsubsection{Definition of frequency estimator}
\label{app:ngram-freq}

For $n$-gram statistics one is often interested in the frequency of a target rather than solely its expected count. Let $\testf(\yvar)=\#\{\text{occurrences of }w\text{ in }\yvar\}$ be the count, the test function of \Cref{tab:ngram}, and {$N_n(\yvar)$ the number of word $n$-gram slots in $\yvar$, i.e., the number of positions in the decoded text of $\yvar$ at which $n$ consecutive words occur (for $n=1$, the number of words). Here a word is a maximal run of ASCII letters, and consecutive words must be separated by a single space, so punctuation ends an $n$-gram.} The frequency of a substring $w$ made of $n$ {such} words is $\estimand_{\mathrm{FR}}\defeq\E[\testf(\yvar)]/\E[{N_n}(\yvar)]${. Since every slot holds exactly one $n$-gram, the counts of all word $n$-grams of order $n$ sum to $N_n(\yvar)$, so, as a function of $w$, $\estimand_{\mathrm{FR}}$ is the marginal distribution over word $n$-grams of order $n$. We estimate it} with the ratio estimators
\begin{equation}
\label{eq:frequency-estimators}
\thfr
\defeq
\frac{\thmc}{\frac{1}{M}\sum_{m=1}^{M}{N_n}(\yvar^{(m)})},
\qquad
\thfr(\ppot)
\defeq
\frac{\hest(\ppot)}{\frac{1}{M}\sum_{m=1}^{M}{N_n}(\yvar^{(m)})},
\end{equation}
where $\thmc$ and $\hest(\ppot)$ are the Monte Carlo and potential estimators of $\E[\testf(\yvar)]$ built from \Cref{tab:ngram}, and the denominator is the Monte Carlo estimate of $\E[{N_n}(\yvar)]$ over the same $M$ generations. Ratio estimators are biased at finite sample size, but the bias vanishes asymptotically under standard regularity conditions \citep{cochran1977sampling}.

\subsubsection{Definition of cost-adjusted variance ratio}
\label{app:ngram-rho}

For each target $w$, we compare the potential estimator against Monte Carlo at
equal sample size and equal compute. For either the expected count or frequency,
we define
\begin{equation}
\label{eq:variance-ratio}
\widehat{\rho}(w)
=
\frac{
\widehat{\Var}\bigl(\text{potential estimator at }M=4{,}000\bigr)
}{
1000\,
\widehat{\Var}\bigl(\text{Monte Carlo estimator at }M=4\times10^6\bigr)
},
\qquad
\widehat{\rho}_{c}(w)
=
{(L_w+1)}\,\widehat{\rho}(w).
\end{equation}
The factor $1000$ rescales the Monte Carlo variance from $4\times10^6$ to
$4{,}000$ samples. This rescaling is exact for expected counts, whose estimator
is a sample mean. For frequencies, which are ratios of sample means, the same $1/M$ variance
scaling holds asymptotically by the multivariate delta method
\citep[\S1.3]{rieder2012robust}.

{The per-target token-evaluation cost proxy is $c_w=1+(L_w-1)+1=L_w+1$: one generation position supplies the first-token probability, $L_w-1$ positions supply the remaining target-token probabilities, and one position supplies the closing-boundary probability. A single-token target therefore has cost multiplier $2$. Under this proxy and the $1/M$ variance scaling above, $\widehat{\rho}_{c}(w)$ estimates the fraction of Monte Carlo compute needed to attain the same variance. For frequency ratios, this comparison does not establish equal mean squared error because finite-sample bias is not included.}

\FloatBarrier
\subsubsection{Observed-prefix potential}
\label{app:ngram-observed}
The $Z^{\mathrm{end}}$ decomposition gives the observed-prefix potential
$\ppot(\xprefix)=\testf(\xprefix)$ of \Cref{tab:ngram-observed}. Its expected increment is
nonzero only at prefixes that already match the substring up to its final token, so the
estimator averages over that token alone. \Cref{tab:ngram-observed-var} reports the median
variance of the resulting estimator relative to Monte Carlo at equal sample size
($M=4{,}000$) for the unigram targets of GPT-2, banded by their occurrences in the
$4{,}000$ samples (very rare $1$--$9$, rare $10$--$99$, moderate $100$--$999$, frequent
$\geq1{,}000$).

\medskip
\noindent
\begin{minipage}[t]{0.53\textwidth}
\setlength{\dtcardextra}{0pt}
\refstepcounter{table}\label{tab:ngram-observed}%
\dtcard{@{}V@{\hspace{5pt}}l@{\hspace{5pt}}V@{}}{%
\toprule
\dtcardhead{Table~\thetable: Substring counts, observed-prefix potential} \\
\midrule
$\testf(\yvar)=\#\{\text{occurrences of }\boldsymbol{w}\text{ in }\yvar\}$\\
\midrule
$\testf(\yvar)=\sum_{t}Z^{\mathrm{end}}_{\boldsymbol{w},t}$, \ $Z^{\mathrm{end}}_{\boldsymbol{w},t}=\1\{\boldsymbol{w}\text{ ends at }t\}$\\
\midrule
{$\ppot^\star(\xprefix)=\E_{\plm}\!\left[\sum_{t=1}^{|\yvar|}Z^{\mathrm{end}}_{\boldsymbol{w},t}\;\middle|\;\yvar\succeq\xprefix\right]$}\\
\midrule
$\ppot(\xprefix)=\testf(\xprefix)=\sum_{t\le|\xprefix|}Z^{\mathrm{end}}_{\boldsymbol{w},t}$\\
\midrule
$\Delta\ppot(a\mid\xprefix)=\1\{\boldsymbol{w}\text{ ends at }|\xprefix|+1\text{ in }\xprefix a\}$\\
\midrule
$\Delta\ppot(\eos\mid\xprefix)=0$\\
\midrule
$\mathbf B\Delta\ppot(\xprefix)=\Pr_{\plm}\bigl[\boldsymbol{w}\text{ ends at }|\xprefix|+1\mid\yvar\succeq\xprefix\bigr]$\\
\bottomrule}
\end{minipage}\hfill
\begin{minipage}[t]{0.45\textwidth}
\vspace{0pt}\centering\footnotesize
\captionof{table}{Variance of the observed-prefix estimator relative to Monte
Carlo at equal sample size ($M=4{,}000$): median over the unigram targets of
GPT-2, banded by their occurrence count in the $4{,}000$ samples (columns).}
\label{tab:ngram-observed-var}
\setlength{\tabcolsep}{2pt}
\begin{tabular}{@{}lcccc@{}}
\toprule
model & $1$--$9$ & $10$--$99$ & $100$--$999$ & $\ge1{,}000$ \\
\midrule
GPT-2 ($854$) & $0.997$ & $0.996$ & $0.993$ & $0.997$ \\
\bottomrule
\end{tabular}
\end{minipage}
\medskip

\subsubsection{Results for expected counts}
\label{app:ngram-counts}

\textbf{Results for expected counts.} \Cref{fig:headline-count} repeats
\Cref{fig:headline} for the expected count $\theta_w=\E[C_w]$ instead of the frequency;
{the results again show the largest gains for rare targets, with some common targets costing more than Monte Carlo at equal variance.}

{\textbf{Half the retained targets reach equal variance on at most $0.80\%$ of the Monte Carlo compute.} \Cref{fig:headline-count}a shows the cost-adjusted variance ratio
$\widehat{\rho}_c$ for the $4{,}149$ of $7{,}647$ targets whose ratio can be estimated with relative
standard error at most $2$ ($3{,}700$ for frequencies), together with approximate per-target
$95\%$ log-ratio confidence intervals. Under the per-target cost proxy, the median ratio is
$0.79\%$ of the Monte Carlo budget, and approximately $76.8\%$ of the retained targets require at most $10\%$.
Replacing each ratio by the upper bound of its $95\%$ interval gives a median of
$1.78\%$ of the Monte Carlo budget; approximately $77.2\%$ of the retained targets have upper bounds at most
$20\%$, and $95.4\%$ remain below parity with Monte Carlo.}

\textbf{The gains are largest for rare targets.} We restrict to the same $2{,}831$
targets as for frequencies, those whose count estimates have nominal RSE at most $20\%$
under a Poisson approximation, equivalently at least $25$ expected occurrences in the
independent $2\times10^6$-sample frequency bank. The compute ratio increases strongly with
{the expected count per document (Spearman $r_s=0.87$), with its median rising from $2.3\times10^{-3}$ in the rarest populated decade to $1.24$ among targets with one to ten expected occurrences per document (\Cref{fig:headline-count}b). For the latter group, the median exceeds parity with Monte Carlo under the cost proxy.}

\textbf{At the same sample size ($M=4{,}000$), the potential estimator yields
substantially more precise count estimates}, with $79.7\%$ of targets reaching RSE
$\leq 20\%$ and $96.5\%$ reaching RSE $\leq 50\%$, compared with $11.4\%$ and $19.0\%$ for
Monte Carlo (\Cref{fig:headline-count}c). The gap grows with token length: for targets of
length $4$ or more, Monte Carlo almost never reaches RSE $\leq 50\%$, whereas the potential
estimator does so for at least $90\%$ of targets in every length group. For $L_w\geq 7$,
Monte Carlo does not observe any target even in the $4\times10^6$-sample reference bank,
while the potential estimator gives a finite estimate for all of them. Every share agrees
with the corresponding share for frequencies to within $0.14$ percentage points.

\begin{figure}[!htb]
\centering
\begin{minipage}[t]{0.33\textwidth}
\centering
{\footnotesize\textbf{(a)} $\widehat{\rho}_c(w)$ with log-ratio 95\% CI}\\[3pt]
\includegraphics[width=\linewidth]{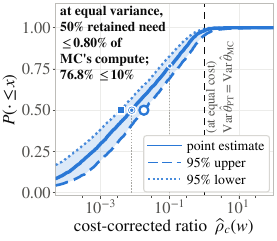}
\end{minipage}\hfill
\begin{minipage}[t]{0.33\textwidth}
\centering
{\footnotesize\textbf{(b)} $\widehat{\rho}_c(w)$ against expected count}\\[3pt]
\includegraphics[width=\linewidth]{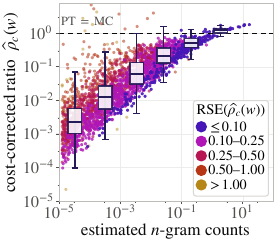}
\end{minipage}\hfill
\begin{minipage}[t]{0.33\textwidth}
\centering
{\footnotesize\textbf{(c)} resolved by token length}\\[3pt]
\includegraphics[width=\linewidth]{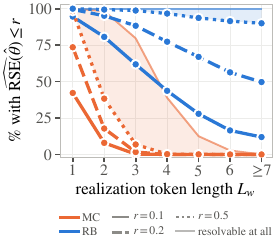}
\end{minipage}
\caption{\footnotesize {\textbf{\Cref{fig:headline} for expected counts.} Cost-adjusted variance ratio $\widehat{\rho}_c(w)$ for $n$-gram count estimation under the per-target cost proxy $L_w+1$. \textbf{(a)} Empirical cumulative distribution over $4{,}149$ targets with bootstrap RSE at most $2$, with approximate $95\%$ log-ratio confidence bounds. \textbf{(b)} Ratio versus estimated expected count per document for the $2{,}831$ targets with at least $25$ expected occurrences in the independent frequency bank. \textbf{(c)} Fraction of targets whose count estimate reaches each RSE threshold, grouped by token length, at the same sample size $M=4{,}000$ for both estimators. Dashed parity lines in (a--b) mark equal variance at equal cost under the proxy; in (c), line styles distinguish RSE thresholds.}}

\label{fig:headline-count}
\end{figure}

\FloatBarrier
\subsubsection{Additional diagnostics}
\label{app:ngram-figures}

\begin{figure}[!t]
\centering
\begin{minipage}[t]{0.26\textwidth}
\centering
{\footnotesize\textbf{(a)} targets and token lengths per order $n$}\\[3pt]
\scriptsize
\setlength{\tabcolsep}{2.3pt}
\renewcommand{\arraystretch}{1.66}
\arrayrulecolor{black}\setlength{\arrayrulewidth}{0.5pt}%
\vspace{1pt}%
\begin{tabular}{|@{\hspace{3pt}}crrrrr@{\hspace{3pt}}|}
\hline
\rule{0pt}{2.1ex}$n$ & $N$ & min & max & mean & med \\
\hline
1 & 869  & 1 & 4  & 1.110 & 1 \\
2 & 1584 & 2 & 5  & 2.128 & 2 \\
3 & 1568 & 3 & 6  & 3.174 & 3 \\
4 & 1391 & 4 & 8  & 4.219 & 4 \\
5 & 1206 & 5 & 9  & 5.266 & 5 \\
6 & 1029 & 6 & 10 & 6.320 & 6 \\
\hline
all & 7647 & -- & -- & -- & -- \\
\hline
\end{tabular}\arrayrulecolor{black}%
\end{minipage}\hfill
\begin{minipage}[t]{0.34\textwidth}
\centering
{\footnotesize\textbf{(b)} share of targets at each $L_w$, per order}\\[3pt]
\includegraphics[width=\linewidth]{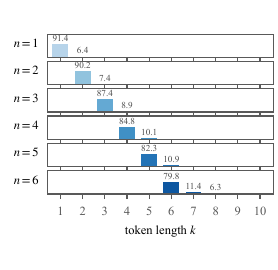}
\end{minipage}\hfill
\begin{minipage}[t]{0.34\textwidth}
\centering
{\footnotesize\textbf{(c)} same, log scale: tails thin alike across $n$}\\[3pt]
\includegraphics[width=\linewidth]{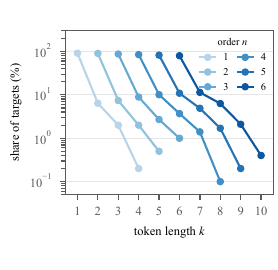}
\end{minipage}
\caption{\footnotesize
\textbf{$n$-gram target inventory and token-length distribution.}
The $n$-gram task provides a controlled way to study targets of different
lengths and, consequently, different marginalization costs, which scale with
the number of GPT-2 tokens $L_w$ in the target.
(a) Number of targets at each order $n$, together with the minimum, maximum,
mean, and median $L_w$ of their canonical realizations.
The inventory contains $7{,}647$ targets in total.
(b) Distribution of token lengths within each order. Increasing $n$ shifts the
inventory toward longer targets, while GPT-2 tokenization also induces
within-order variation in $L_w$.
(c) The same distributions on a log scale, showing that each order additionally
contains a progressively thinning tail of targets requiring more tokens, and
therefore more costly marginalization.
}
\label{fig:ngram-inventory}
\end{figure}

\begin{figure}[!htb]
\centering
\begin{minipage}[t]{0.48\textwidth}
\centering
\includegraphics[width=\linewidth]{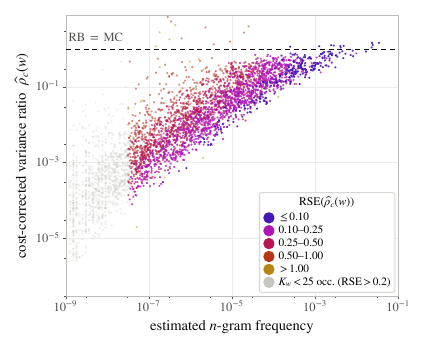}
\end{minipage}\hfill
\begin{minipage}[t]{0.48\textwidth}
\centering
\includegraphics[width=\linewidth]{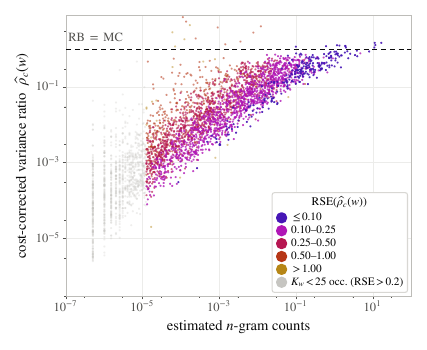}
\end{minipage}
\caption{\footnotesize
\textbf{{Cost-adjusted variance ratio against target rarity before the occurrence filter.}}
{Both panels show the frequency-estimator variance ratio, against estimated $n$-gram frequency (left) and expected count per document (right), before applying the occurrence filter to the independently estimated $x$-axis quantity.}
In the main figure, we retain only targets whose independent frequency estimate
has RSE at most \(20\%\), corresponding to at least \(25\) expected occurrences
in the estimation sample, to restrict the analysis to more reliable estimates.
{Here, we show all $3{,}835$ targets with positive estimated frequency and a finite positive variance ratio; this remains a subset of the $7{,}647$-target inventory.}}
\label{fig:ngram-rho-both-axes}
\end{figure}

\begin{table}[!htb]
\centering
\caption{
\textbf{Variance reduction is strongest for rare targets, even within fixed
$n$-gram order and token length.}
The $2{,}831$ targets of \Cref{fig:headline}b are grouped by their independently
estimated frequency $\thfr$. Each entry reports the median cost-adjusted
variance ratio $10^3\widehat{\rho}_c$ within that frequency bin; lower values
indicate larger gains over Monte Carlo.
\textbf{(a)} The trend across all targets.
\textbf{(b--c)} The same comparison within fixed $n$-gram order and canonical
token length $L_w$, showing that the rarity trend persists after controlling
for target length and marginalization cost.
The bottom rows give the number of targets in each column and Spearman's
correlation $r_s$ between $\thfr$ and $\widehat{\rho}_c$.
Cells with fewer than $8$ targets are omitted.
}
\label{tab:ngram-rarity-strata}

\footnotesize
\setlength{\tabcolsep}{3.8pt}
\begin{tabular}{@{}l rr @{\hspace{10pt}} rrrr @{\hspace{10pt}} rrrr@{}}
\toprule
& \multicolumn{2}{c}{\textbf{(a)} all targets}
& \multicolumn{4}{c}{\textbf{(b)} $n$-gram order}
& \multicolumn{4}{c}{\textbf{(c)} token length $L_w$} \\
\cmidrule(lr){2-3}
\cmidrule(lr){4-7}
\cmidrule(l){8-11}

frequency $\thfr$
& \# targets
& $10^3\widehat{\rho}_c$
& $1$ & $2$ & $3$ & $4$
& $1$ & $2$ & $3$ & $4$ \\
\midrule

$[10^{-8},10^{-7})$ & 354 & {1.5} & {2.0} & {1.5} & {1.5} & {1.3} & -- & {1.6} & {1.5} & {1.3} \\

$[10^{-7},10^{-6})$ & 701 & {5.8} & {6.5} & {6.6} & {5.3} & {5.1} & -- & {6.6} & {5.4} & {5.1} \\

$[10^{-6},10^{-5})$ & 725 & {28} & {35} & {28} & {25} & {17} & {42} & {26} & {25} & {17} \\

$[10^{-5},10^{-4})$ & 684 & {130} & {160} & {110} & {85} & -- & {160} & {110} & {85} & -- \\

$[10^{-4},10^{-3})$ & 299 & {310} & {330} & {260} & -- & -- & {330} & {260} & -- & -- \\

$[10^{-3},10^{-1})$ & 68 & {710} & {740} & {530} & -- & -- & {740} & {530} & -- & -- \\

\midrule

\textit{targets}
& 2{,}831 & --
& 865 & 1{,}295 & 537 & 115
& 794 & 1{,}307 & 587 & 124 \\

\textit{Spearman $r_s$} & -- & {0.85} & {0.72} & 0.80 & 0.73 & 0.69 & 0.67 & 0.78 & 0.74 & 0.69 \\

\bottomrule
\end{tabular}
\end{table}

The variance ratios of \Cref{fig:headline} are point estimates, so we also ask how
many targets are separated from parity once their own estimation error is taken
into account. \Cref{tab:rse-cutoff-boundary} reports this across cumulative
relative-standard-error subsets.

\begin{table}[!htb]
\centering
\caption{%
{\textbf{A majority of targets in each precision subset retain a variance advantage under the cost proxy.}
Among targets with positive estimated frequency and a finite positive variance ratio, each row retains those with bootstrap RSE below the indicated cutoff $s$.
We report the number of retained targets and the percentage satisfying
$\widehat{\rho}_c<1/(1+1.96s)$, a sufficient condition for the approximate normal interval
$\widehat{\rho}_c[1\pm1.96\,\mathrm{RSE}]$ to lie below parity.
These are conservative shares for the normal interval, rather than the log-ratio intervals in \Cref{fig:headline,fig:headline-count}.}
}
\label{tab:rse-cutoff-boundary}

\begin{tabular}{@{}l rr rr@{}}
\toprule
& \multicolumn{2}{c}{Expected count}
& \multicolumn{2}{c}{Frequency} \\
\cmidrule(lr){2-3}\cmidrule(l){4-5}
Max.\ RSE
& \# targets & 95\% CI below $1$
& \# targets & 95\% CI below $1$ \\
\midrule
5\%  & 84      & {61.9\%} & 41      & {70.7\%} \\
10\% & 317     & {84.2\%} & 240     & {89.6\%} \\
20\% & 1{,}048 & {92.2\%} & 953     & {95.2\%} \\
50\% & 2{,}661 & {93.0\%} & 2{,}583 & {94.7\%} \\
\bottomrule
\end{tabular}
\end{table}\FloatBarrier
\subsection{Event Analysis of \citet{dorman2026rare}}
\label{app:exp-events}

\subsubsection{The Automated Readability Index}
\label{app:ari}
We extend the experiments of \citep{dorman2026rare} that estimate the Automated Readability Index. For a string $\yvar$ with $c(\yvar)$
characters, $w(\yvar)$ words, and $s(\yvar)$ sentences,
\begin{equation}
\label{eq:ari}
\mathrm{ARI}(\yvar)
=
4.71\,\frac{c(\yvar)}{w(\yvar)}+0.5\,\frac{w(\yvar)}{s(\yvar)}-21.43,
\end{equation}
a score calibrated
so that its value roughly matches the US school grade needed to read the text.
The events $D_j$ are bins of this score, and the test functions are the
indicators $\1\{\mathrm{ARI}(\yvar)\in D_j\}$.
\paragraph{Experimental setup.}
The experiments follow the authors' setup. The model TinyStories-8M continues a fixed prompt for
$100$ tokens, and the target is the probability that the completion's
$\operatorname{ARI}$ falls in each of $80$ bins of width $0.29$ spanning
$[-8,15]$. The model is tilted by $\exp(-\lambda\operatorname{ARI})$ at $19$
values of $\lambda$; each tilt is sampled by transition-path sampling ($200$
chains of $32{,}400$ retained states) and complemented by $200{,}000$ direct
samples, and the resulting $6.68\times10^{6}$ states are combined with MBAR
weights, which we retain unchanged. Folds are whole chains, stratified by tilting schedule.

\subsubsection{Potential definitions}
\label{app:learned-potentials}

The events are the $J=80$ readability bins $D_j$ of \Cref{app:ari}. For each bin
we learn a potential $\ppot_j(\xprefix;\boldsymbol{\eta})=g_j(\boldsymbol{\xi}(\xprefix);\boldsymbol{\eta})$
of the form of \cref{eq:learned-potential}, approximating the oracle
$\Pr_{\plm}[\mathrm{ARI}(\yvar)\in D_j\mid\yvar\succeq\xprefix]$. The prefix
features $\boldsymbol{\xi}(\xprefix)$ collect the character, word, and sentence
counts of $\xprefix$ together with its position in the generation window. We
consider two models, a logistic model that predicts bin probabilities directly
and a mixture model that forecasts the counts underlying the ARI; both are
fitted by the cross-fitting scheme introduced in \Cref{sec:app-events}, with {potential parameters fitted on the training folds using the globally estimated MBAR weights.}

\paragraph{Logistic potential.}
The logistic model predicts bin probabilities from the prefix features
directly,
\begin{equation}
\ppot^{\mathrm{log}}_j(\xprefix;\boldsymbol{\eta})
=
\frac{
\widehat p_j
\exp\!\left(\boldsymbol{\beta}_j^\top\boldsymbol{\xi}(\xprefix)\right)
}{
\sum_{j'=1}^{J}
\widehat p_{j'}
\exp\!\left(\boldsymbol{\beta}_{j'}^\top\boldsymbol{\xi}(\xprefix)\right)
},
\qquad
\boldsymbol{\eta}=(\boldsymbol{\beta}_1,\ldots,\boldsymbol{\beta}_J),
\end{equation}
where $\widehat p_j$ is the marginal probability of bin $D_j$ estimated from
the training data.  The feature vector $\boldsymbol{\xi}(\xprefix)$ here comprises
11 features: an intercept, the normalized prefix position and its square, the inverse
remaining length and its inverse square root, the character, word, and
sentence counts, and the interactions between each count and the normalized
position, so that the counts can play different roles at different stages of
generation. We fit $\boldsymbol{\eta}$ by minimizing MBAR-weighted
cross-entropy.

\paragraph{Mixture-model potential.}
At each prefix, we linearly extrapolate the character and word counts in
$\boldsymbol{\xi}(\xprefix)$ to the terminal length. We separately model the
number $k$ of additional sentence endings, since a single further sentence
break can change the ARI substantially. Each value of $k$ gives a predicted
ARI $\mu_k(\xprefix)$, around which we place a Gaussian with standard
deviation $\sigma_k(\xprefix)$. Combining these components with their
predicted probabilities $\pi_k(\xprefix)$ gives a Gaussian mixture law for the
terminal $\mathrm{ARI}(\yvar)$ of the completed string, and the potential is the
mass this law assigns to the bin $D_j=[l_j,u_j)$, available in closed form
through the standard normal distribution function $\Phi$:
\begin{equation}
\begin{aligned}
q_{\mathrm{mix}}(\,\cdot\mid\xprefix)
&=
\sum_k \pi_k(\xprefix)\,
\mathcal{N}\!\left(\mu_k(\xprefix),\sigma_k^2(\xprefix)\right),
\\
\ppot^{\mathrm{mix}}_j(\xprefix;\boldsymbol{\eta})
&=
q_{\mathrm{mix}}(D_j\mid\xprefix)
=
\sum_k \pi_k(\xprefix)\left[
\Phi\!\left(\frac{u_j-\mu_k(\xprefix)}{\sigma_k(\xprefix)}\right)
-\Phi\!\left(\frac{l_j-\mu_k(\xprefix)}{\sigma_k(\xprefix)}\right)
\right].
\end{aligned}
\end{equation}
Here $\boldsymbol{\eta}$ collects the extrapolation rates, mixture weights, and Gaussian widths. We fit these parameters on the combined direct and transition-path samples, giving each retained sample equal weight. To recover target-model bin probabilities, we multiply the fitted mixture density by the MBAR weight as a function of terminal ARI, integrate the resulting weighted density over each bin using fine-grid quadrature, and normalize once across the full mixture. We cap the number of additional sentence endings at $14$, a heuristic upper bound for the 100-token completions used in our experiments.
\FloatBarrier
\subsubsection{Potential selection algorithm}
\label{app:events-selection}

For each potential $\ppot$, we use the control-variate estimator of \cref{eq:is-potential-cv}.
The coefficient $\alpha$ may either be fixed a priori or estimated from data to minimize variance. In particular, $\alpha=1$ yields the potential estimator, while $\alpha=0$ recovers the baseline estimator regardless of the potential. \Cref{alg:potential-selection} gives the nested cross-fitting procedure used in \Cref{sec:app-events} to select both the potential and, when learned, its coefficient using held-out folds.

\begingroup\setlength{\intextsep}{4pt}
\begin{algorithm}[!htb]
\caption{Potential selection by nested cross-fitting}
\label{alg:potential-selection}
\small
\begin{algorithmic}[1]

\Require Sampled strings $\mathcal D$ with stored next-token
         distributions; target function $f$.
\Require $P$ candidate potentials, each fixed or learned,
         with fitting procedures for learned parameters.
\Require For each candidate, a fixed coefficient $\alpha$
         (default: $1$) or a coefficient-fitting procedure.
\Require Outer and inner fold counts $K$ and $L$.
\Ensure Estimate $\widehat{\theta}$ of $\mathbb E_p[f(Y)]$.

\State Split $\mathcal D$ into $K$ equal-sized outer folds
       $\mathcal D_1,\ldots,\mathcal D_K$.

\For{$k=1,\ldots,K$}
    \State $\mathcal T \gets \mathcal D\setminus\mathcal D_k$
           \Comment{Outer training data}
    \State Split $\mathcal T$ into $L$ equal-sized inner folds
           $\mathcal V_1,\ldots,\mathcal V_L$.

    \For{$j=1,\ldots,P$}
        \For{$\ell=1,\ldots,L$}
            \State $\mathcal A \gets
                   \mathcal T\setminus\mathcal V_\ell$

            \State Fit candidate $j$'s potential on $\mathcal A$,
                   if needed, to obtain $\phi$.
            \State Set $\alpha$ to its fixed value,
                   or fit it on $\mathcal A$.

            \State Evaluate $\widehat{\theta}(y_i;\phi,\alpha)$
                   for every $y_i\in\mathcal V_\ell$.
            \State $v_{j\ell}\gets$
                   sample variance of these estimates.
        \EndFor

        \State $s_j\gets
               \frac{1}{L}\sum_{\ell=1}^{L}v_{j\ell}$
               \Comment{Validation score}
    \EndFor

    \State $j^\star\gets\arg\min_{j=1,\ldots,P}s_j$

    \State Refit the selected potential on $\mathcal T$,
           if needed, to obtain $\phi_k$.
    \State Set $\alpha_k$ to the selected candidate's fixed value,
           or refit it on $\mathcal T$.

    \State $\displaystyle
           \widehat{\theta}_k\gets
           \frac{1}{|\mathcal D_k|}
           \sum_{y_i\in\mathcal D_k}
           \widehat{\theta}(y_i;\phi_k,\alpha_k)$
           \Comment{Held-out estimate}
\EndFor

\State \Return $\displaystyle
       \widehat{\theta}
       =\frac{1}{K}\sum_{k=1}^{K}\widehat{\theta}_k$

\end{algorithmic}
\end{algorithm}

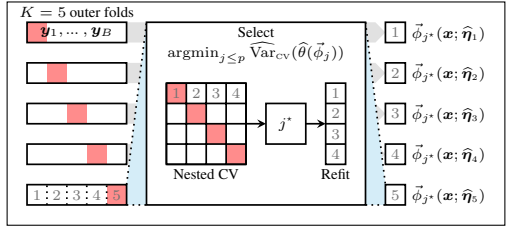
\begin{figure}[!htb]
\centering
\begin{minipage}[c]{0.50\textwidth}
\centering
\pgfmathsetmacro{\gridW}{262.3*6/7} 
\pgfmathsetmacro{\gridH}{114.6}     
\pgfmathsetmacro{\gridCols}{6}
\pgfmathsetmacro{\gridRows}{3}
\pgfmathsetmacro{\gridSub}{4}       
\pgfmathsetmacro{\cellW}{\gridW/(\gridCols*\gridSub)}
\pgfmathsetmacro{\cellH}{\gridH/(\gridRows*\gridSub)}
\pgfmathsetmacro{\gridLW}{0.5}      
\scalebox{0.8}{%
\begin{tikzpicture}[x=\cellW pt,y=\cellH pt,
  box/.style={draw,line width=0.8pt},
  held/.style={fill=red!45,draw=none},
  arr/.style={->,line width=0.8pt,>=stealth},
  lab/.style={font=\scriptsize,inner sep=1pt},
  fan/.style={draw,line width=0.8pt,dotted},
  fanfill/.style={fill=skyblue,fill opacity=0.35,draw=none},
  band/.style={fill=black,fill opacity=0.12,draw=none}]
\pgfmathsetmacro{\gridExtra}{4}    
\pgfmathsetmacro{\gridCut}{3}      
\pgfmathsetmacro{\nx}{\gridCols*\gridSub+\gridExtra-\gridCut}
\pgfmathsetmacro{\gridTrim}{1}     
\pgfmathsetmacro{\ny}{\gridRows*\gridSub-\gridTrim}
\draw[line width=\gridLW pt] (0,0) rectangle (\nx,\ny);
\begin{scope}[shift={(-\gridCut,0)}]  
\foreach \k in {1,...,4}{
  \pgfmathsetmacro{\yb}{11-2*\k}
  \fill[band] (9,\yb) -- (21.5,\yb) -- (22,\yb+0.5) -- (21.5,\yb+1) -- (9,\yb+1) -- cycle;
}
\definecolor{skyblue}{RGB}{135,206,235}
\fill[white] (9,1) -- (9,2) -- (10,10) -- (10,1) -- cycle;
\fill[white] (21,1) -- (21,10) -- (22,2) -- (22,1) -- cycle;
\fill[fanfill] (9,1) -- (9,2) -- (10,10) -- (10,1) -- cycle;
\fill[fanfill] (21,1) -- (21,10) -- (22,2) -- (22,1) -- cycle;
\draw[fan] (9,2) -- (10,10);  \draw[fan] (9,1) -- (10,1);
\draw[fan] (21,10) -- (22,2); \draw[fan] (21,1) -- (22,1);
\node[lab] at (6.5,10.5) {$K=5$ outer folds};
\foreach \k in {1,...,5}{
  \pgfmathsetmacro{\yb}{11-2*\k}
  \fill[held] (3+\k,\yb) rectangle (4+\k,\yb+1);
  \draw[box] (4,\yb) rectangle (9,\yb+1);
}
\node[lab] at (6.5,9.5) {$\str_1,\ldots,\str_B$};  
\foreach \x in {5,6,7,8}{ \draw[line width=0.8pt,dotted] (\x,1) -- (\x,2); }
\foreach \i in {1,...,5}{ \node[lab,text=black!55] at (3.5+\i,1.5) {$\i$}; }
\draw[box,fill=white] (10,1) rectangle (21,10);  
\node[lab] at (15.5,9.5) {Select};
\node[lab] at (15.5,8.5) {$\argmin_{j\le p}\widehat{\Var}_{\scalebox{0.7}{$\scriptstyle\mathrm{CV}$}}(\hest(\ppot_j))$};
\foreach \i in {1,...,4}{ \fill[held] (10+\i,7-\i) rectangle (11+\i,8-\i); }
\draw[line width=0.8pt,step=1] (11,3) grid (15,7);
\draw[box] (11,3) rectangle (15,7);
\foreach \i in {1,...,4}{ \node[lab,text=black!55] at (10.5+\i,6.5) {$\i$}; }
\node[lab] at (13,2.5) {Nested CV};
\node[lab] at (19.5,2.5) {Refit};
\draw[arr] (15,5) -- (16,5);
\draw[box] (16,4) rectangle (18,6);
\node[lab] at (17,5) {$j^{\scalebox{0.6}{$\star$}}$};
\draw[arr] (18,5) -- (19,5);
\draw[box] (19,3) rectangle (20,7);
\foreach \j in {4,5,6}{ \draw[line width=0.8pt] (19,\j) -- (20,\j); }
\foreach \i in {1,...,4}{ \node[lab,text=black!55] at (19.5,7.5-\i) {$\i$}; }  
\foreach \k in {1,...,5}{
  \pgfmathsetmacro{\yb}{11-2*\k}
  \draw[box] (22,\yb) rectangle (23,\yb+1);
  \node[lab,text=black!55] at (22.5,\yb+0.5) {$\k$};
  \node[lab,anchor=west] at (23.2,\yb+0.5) {$\ppot_{j^{\scalebox{0.6}{$\star$}}}(\xprefix;\widehat{\boldsymbol{\eta}}_{\scalebox{0.8}{$\scriptstyle\k$}})$};
}
\end{scope}
\end{tikzpicture}}
\end{minipage}\hfill
\begin{minipage}[c]{0.47\textwidth}
\caption{\footnotesize Nested cross-fitting, our in-sample selection among $p$
candidate potentials (\Cref{alg:potential-selection}). Each of $K=5$ outer
folds is held out in turn (red); on the remaining folds the candidates are
compared by inner cross-validation, the winner $j^\star$ is refitted on all
remaining data and applied to the held-out fold.
\Timv{How is the number of folds $K$ chosen and how much does it affect the results?  Would leave-one-out be any better/different?  Is the issue the computational cost of fitting that many classifiers too high?}
}
\label{fig:nested-cv-scheme}
\end{minipage}
\end{figure}
\endgroup

\clearpage

\subsubsection{Compute costs and additional results}
\label{app:events-results}

\paragraph{Compute.}
All potential fitting, scoring, and selection run on CPUs, with no GPU computation beyond sampling. The logistic potential has $880$ coefficients ($11$ features across $80$ bins), while the mixture model stores empirical rates, mixture weights, and residual widths. We use five outer folds, with the four remaining folds serving as the inner folds. Each inner fit uses three folds: exchanging the outer held-out fold and the inner validation fold leaves the training set unchanged. Thus, the $5\times4=20$ inner evaluations require only $\binom{5}{2}=10$ distinct fits, which we reuse across outer splits. Together with $5$ outer refits and $1$ full-data fit, this gives $16$ parameter sets per model. Fitting these takes $14$ minutes for logistic regression and $4$ minutes for the mixture model on a single $32$-core node. Scoring both potentials using the stored next-token distributions requires ten parallel jobs, each using $32$ cores for $50$--$95$ minutes.

\paragraph{Additional results.}
\Cref{fig:nested-cv-scheme} sketches the selection procedure and \Cref{fig:nested-cv-full} shows the variance of the selected estimator over the bins on which MBAR's own variance is well measured\claude{, \Cref{fig:nested-cv-candidates} the same view for each single candidate}, \Cref{fig:nested-allbins} extends the comparison to all $64$ populated bins with the accompanying diagnostics, and \Cref{tab:rea-nested-full} lists the chain-level variance, sample-size efficiency, and MBAR-equivalent sample size of every estimator.

\begin{figure}[!htb]
\centering
\includegraphics[height=134.05pt]{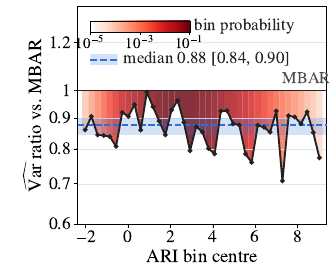}
\vspace{-4pt}
\caption{\footnotesize Variance ratio of the selected potential on the event-analysis
experiment: chain-level variance of the cross-fitted estimator relative to
MBAR, per ARI bin, over \claude{the $39$ bins $20$--$58$: the $38$ bins on which MBAR's own variance is itself
measured to within $20\%$, plus bin $57$ (at $25\%$), kept so that the range is contiguous}; shading gives the bin probability, the dashed line
the median ($0.88$) and the band its $95\%$ interval. All $64$ populated bins
are in \Cref{fig:nested-allbins}.}
\label{fig:nested-cv-full}
\end{figure}

\begin{figure}[!htb]
\centering
\begin{minipage}[t]{0.46\textwidth}\centering{\footnotesize\textbf{(a)} terminal potential}\\[1pt]\includegraphics[height=134.05pt]{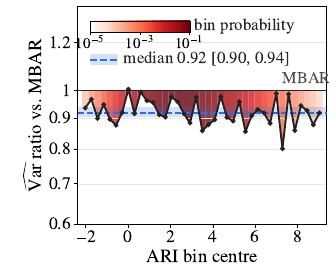}\end{minipage}\hfill
\begin{minipage}[t]{0.46\textwidth}\centering{\footnotesize\textbf{(b)} terminal potential with $\alpha$}\\[1pt]\includegraphics[height=134.05pt]{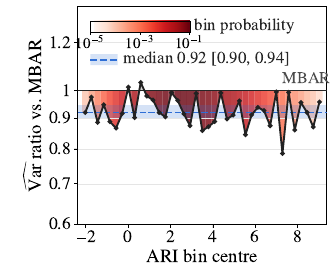}\end{minipage}\\[6pt]
\begin{minipage}[t]{0.46\textwidth}\centering{\footnotesize\textbf{(c)} logistic model}\\[1pt]\includegraphics[height=134.05pt]{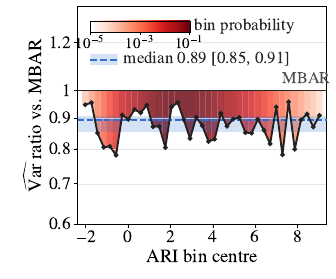}\end{minipage}\hfill
\begin{minipage}[t]{0.46\textwidth}\centering{\footnotesize\textbf{(d)} mixture model}\\[1pt]\includegraphics[height=134.05pt]{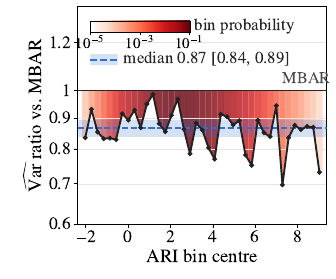}\end{minipage}
\vspace{-4pt}
\caption{\footnotesize \claude{The same view as \Cref{fig:nested-cv-full} for each single candidate, each refitted without the held-out fold and applied to every bin: chain-level variance relative to MBAR per ARI bin over the same $39$ bins ($38$ by the $20\%$ rule plus bin $57$), with the bin probability as shading and the candidate's own median over these bins with its $95\%$ interval. The terminal potential and its calibrated version remove $8\%$, the logistic model $11\%$, and the mixture model $13\%$, against $12\%$ for the selection, which picks the mixture model in four folds and the logistic model in one.}}
\label{fig:nested-cv-candidates}
\end{figure}

\begin{figure}[!htb]
\centering
\includegraphics[width=\textwidth]{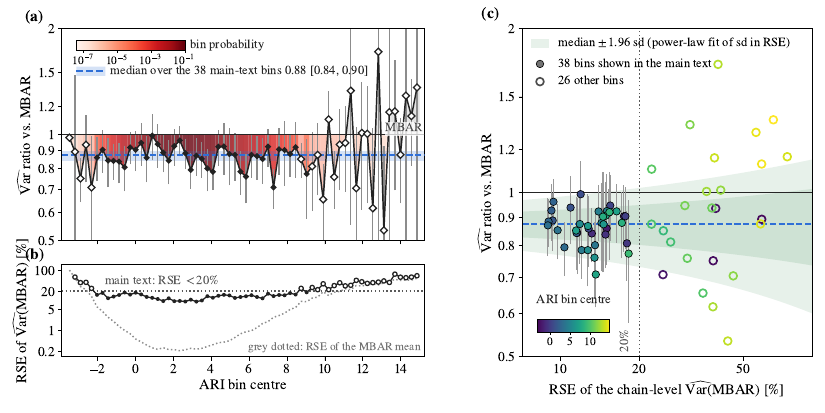}
\caption{\footnotesize Variance reduction across ARI bins. \textbf{(a)}  variance of the cross-fitted estimator relative to MBAR for all $64$ populated bins. Filled markers denote the $38$ bins for which the MBAR variance is estimated with RSE below $20\%$; over these bins, the median variance ratio is $0.88$ with a $95\%$ bootstrap interval of $[0.84,0.90]$\claude{, and over all $64$ populated bins it is $0.88$ with interval $[0.86,0.91]$}. Ratios in the remaining tail bins are much noisier because the MBAR variance itself is poorly estimated. Whiskers show $\pm1.96$ bootstrap standard deviations from $1{,}000$ paired chain-bootstrap replicates. \textbf{(b)} RSE of the MBAR variance and mean estimates across bins, showing that uncertainty in the variance estimate grows sharply in the tails. \textbf{(c)} Variance ratios against the RSE of the MBAR variance estimate. Among the $38$ reliable bins, the ratio shows no systematic dependence on baseline uncertainty (Spearman $\rho=0.05$, $p=0.78$), indicating that the observed reduction is not driven by noise in the MBAR variance estimate.}
\label{fig:nested-allbins}
\end{figure}

\begin{table}[!htb]
\centering\small\setlength{\tabcolsep}{4pt}
\begin{tabular}{l c c c}
\toprule
Estimator & Var.\ ratio vs.\ MBAR [95\% CI] & Sample-size eff.\ [95\% CI] & MBAR-equiv.\ samples ($\times10^6$) [95\% CI] \\
\midrule
Nested selection & 0.88 [0.86, 0.91] & 1.14 [1.10, 1.17] & 7.58 [7.34, 7.80] \\
Terminal & 0.92 [0.90, 0.94] & 1.09 [1.06, 1.11] & 7.29 [7.09, 7.40] \\
Terminal, $\alpha$ & 0.92 [0.90, 0.95] & 1.09 [1.06, 1.11] & 7.27 [7.05, 7.40] \\
Logistic & 0.90 [0.86, 0.91] & 1.12 [1.09, 1.16] & 7.46 [7.31, 7.75] \\
Mixture model & 0.87 [0.84, 0.90] & 1.15 [1.11, 1.19] & 7.69 [7.43, 7.92] \\
\bottomrule
\end{tabular}
\caption{Event analysis: chain-level variance of each estimator relative to MBAR, median over all $64$ populated ARI bins, with $95\%$ intervals from $1{,}000$ paired chain-bootstrap replicas (fitted potentials, selections, and coefficients held fixed). Sample-size efficiency: median over bins of $\mathrm{Var}(\mathrm{MBAR})/\mathrm{Var}(\text{estimator})$; MBAR-equivalent samples: that factor times the actual sample size $N=6.68\times10^{6}$ states.}
\label{tab:rea-nested-full}
\end{table}

\end{document}